\documentclass{article}
\usepackage{arxiv_preprint,times}

\usepackage{amsmath,amsfonts,bm}

\def\eqref#1{equation~\ref{#1}}

\def\1{\bm{1}}

\DeclareMathAlphabet{\mathsfit}{\encodingdefault}{\sfdefault}{m}{sl}
\SetMathAlphabet{\mathsfit}{bold}{\encodingdefault}{\sfdefault}{bx}{n}

\usepackage{hyperref}
\usepackage{url}
\usepackage{graphicx}
\usepackage{booktabs}
\usepackage{amsmath,amssymb,amsthm}

\usepackage{subcaption}
\usepackage{multirow}
\usepackage{xcolor}
\usepackage{xspace}
\usepackage{wrapfig}
\usepackage{colortbl}

\definecolor{taskblue}{RGB}{235,244,252}

\newif\ifnotes
\notesfalse
\ifnotes
  \ifXeTeX
    \usepackage{xeCJK}
  \fi
\fi

\ifnotes
  \long\def\todo#1{{\color{red} [TODO: #1]}}
  \long\def\note#1{{\color{blue} [note: #1]}}
\else
  \long\def\todo#1{}
  \long\def\note#1{}
\fi

\ifnotes
  
\else
  
\fi

\title{Beyond Attention Masks:\\ Instruction Anchoring for Efficient\\ In-Context Diffusion Generation}

\def\realauthors{%
  \normalfont
  Yangshuai Liu\textsuperscript{1,*\ensuremath{\ddagger}},
  Zheming Li\textsuperscript{2,*},
  Jiaao Li\textsuperscript{2},
  Kang He\textsuperscript{2},
  Ziliang Lai\textsuperscript{2},
  \\[0.4ex]
  Zhitai Liu\textsuperscript{1,\ensuremath{\dagger}},
  Chengru Song\textsuperscript{2,\ensuremath{\dagger}}
  \\[0.8ex]
  \normalfont
  \textsuperscript{1}Harbin Institute of Technology
  \qquad
  \textsuperscript{2}KlingAI Research
  \\[0.5ex]
  \normalfont
  \texttt{yangshuai.liu@stu.hit.edu.cn,\ lizheming@klingai.com}
  \\
  \normalfont
  \texttt{ztliu@hit.edu.cn,\ songchengru@klingai.com}
}

\newcommand{\authorfootnotes}{%
  \begingroup
    \renewcommand{\thefootnote}{\fnsymbol{footnote}}%
    \footnotetext[1]{Equal contribution.}%
    \footnotetext[2]{Corresponding authors.}%
    \footnotetext[3]{Work done during an internship at KlingAI Research.}%
  \endgroup
}

\author{\realauthors}

\begin{document}
\hypersetup{
  pdftitle={Beyond Attention Masks: Instruction Anchoring for Efficient In-Context Diffusion Generation},
  pdfauthor={Yangshuai Liu, Zheming Li, Jiaao Li, Kang He, Ziliang Lai, Zhitai Liu, Chengru Song}
}
\maketitle
\authorfootnotes
\fancyhead{}

\begin{abstract}
In-context diffusion transformers concatenate instruction, target, and
reference tokens into a single sequence for joint attention. Reference-side
computation must be repeated at every denoising step, with computation growing
rapidly as more references are added. Decoupling the reference tokens from the
target enables exact K/V reuse across denoising steps. This isolation, however,
prevents the reference tokens from attending to the instruction, degrading
instruction following and reference fidelity. This trade-off cannot be
resolved by attention mask design alone.
We introduce \textbf{AnchorCache}, a parameter-free token layout and attention
mask co-design that inserts static text anchors. The anchors condition the
reference representations on the instruction during cache construction, after
which the resulting reference K/V are reused exactly across denoising steps.
This structural conversion initially degrades generation quality. We recover
the lost performance with teacher-forced velocity distillation, followed by a
short on-policy stage that queries the teacher at student-visited states. To
our knowledge, this is the first use of on-policy distillation for architectural
recovery in diffusion models.
Across benchmarks spanning image, speech, and video generation, AnchorCache
matches full-attention quality. Its efficiency gains increase with the
reference context, reaching a $6.40\times$ speedup in DiT inference.
\end{abstract}

\section{Introduction}
\label{sec:intro}

Diffusion transformers increasingly adopt in-context conditioning to
integrate heterogeneous inputs such as text, images, audio, and video. In
such models, instruction, target, and reference tokens are concatenated into
a single sequence for joint attention
\citep{peebles2023dit,esser2024sd3,liu2025cobra}. Although the reference
inputs remain unchanged across denoising steps, their layerwise
representations depend on the current target through joint attention. This
target dependency requires repeating reference-side computation at every
denoising step, with the cost growing rapidly as more references are added.
At $1024^2$ resolution, for example, a single image reference contributes
roughly $4$k tokens.

A natural way to avoid this repeated computation is to decouple the reference
tokens from the target. This makes their layerwise K and V independent of the
current denoising step, allowing them to be computed once and reused exactly
throughout denoising
\citep{tan2025ominicontrol2,zhang2025easycontrol}. However, existing cacheable
designs enforce this independence by preventing the reference tokens from
attending to the instruction states. Without instruction access, the reference
representations are less able to capture instruction-relevant content,
degrading instruction following and reference fidelity.

The difficulty is that, under full joint attention, the instruction states
are themselves contextualized by the evolving target. Allowing reference
tokens to attend to these states introduces an indirect dependency on the
target, making the reference K/V step-dependent. We therefore need a separate
instruction-conditioning path that remains independent of the evolving target.

To this end, we introduce \textbf{AnchorCache}. Rather than conditioning the
reference branch on the live instruction states, AnchorCache introduces a
static copy of the instruction that is isolated from the evolving target.
This allows the reference branch to incorporate instruction information
while preserving exact K/V reuse across denoising steps. The construction
introduces no learnable parameters and retains a regular attention layout
compatible with optimized implementations such as FlashAttention.

\begin{figure*}[tbp]
  \centering
  \includegraphics[width=\textwidth]{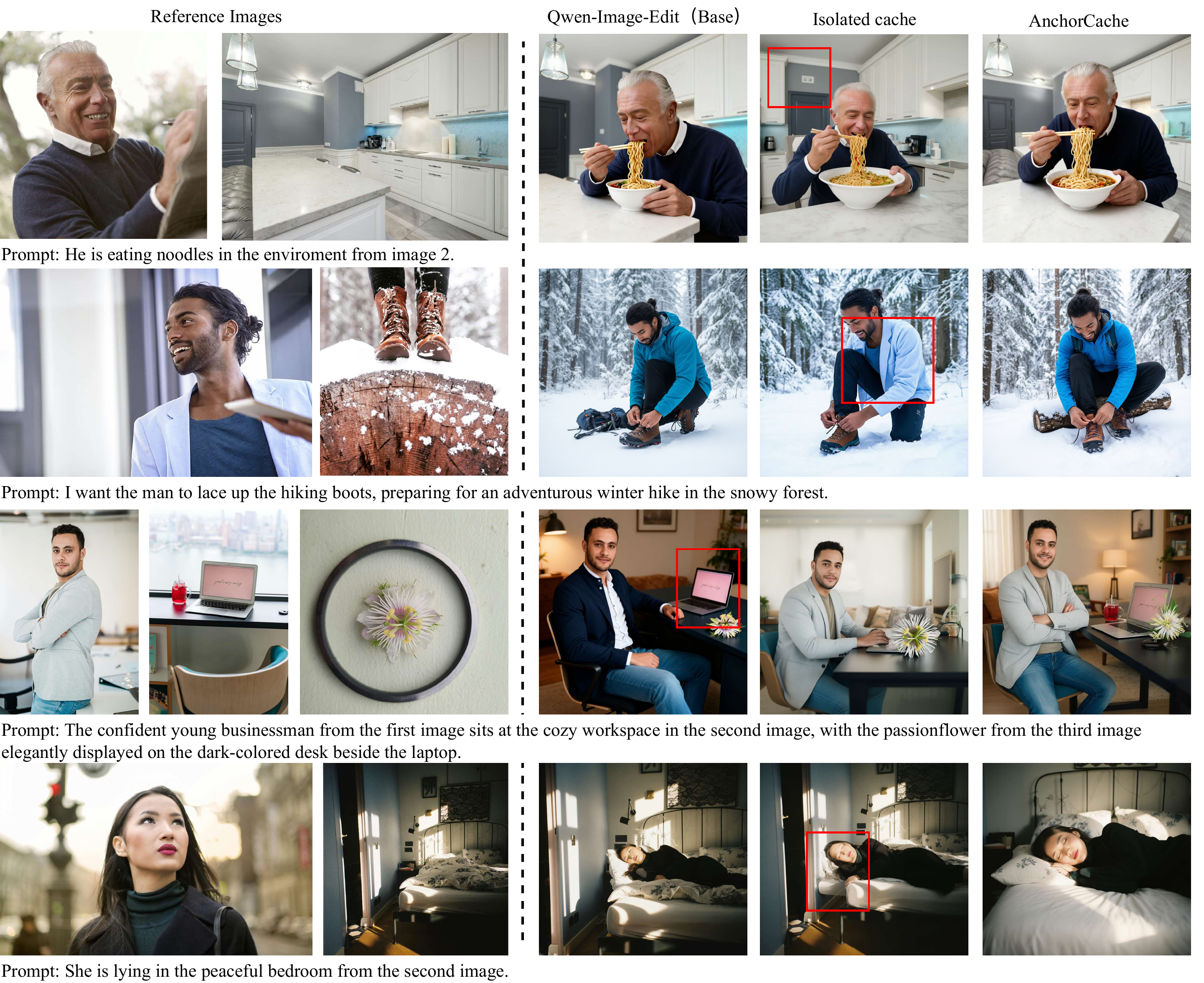}
  \caption{
  \textbf{Qualitative comparison on OmniContext.}
  Each row shows the reference images and outputs from the full-attention base
  model, isolated cache, and AnchorCache. Red boxes highlight identity
  inconsistencies.
  }
  \label{fig:qualitative}
\end{figure*}

Converting a pretrained full-attention model to this cacheable structure
changes its information flow and initially degrades generation quality.
We first recover performance through teacher-forced velocity distillation
at data-derived states. Since this supervision does not directly cover the
states encountered along the student's inference trajectory, we further
introduce a short on-policy stage that queries the full-attention teacher
at student-visited states. This two-stage procedure targets both the initial
quality loss from architectural conversion and residual errors along the
student's denoising trajectory.

We evaluate AnchorCache across image, speech, and video generation, including
both conventional multi-step and few-step diffusion settings. Across
task-specific benchmarks, AnchorCache matches the quality of the corresponding
full-attention models. 

Our contributions are summarized as follows:
\begin{itemize}
  \item We introduce \textbf{AnchorCache}, a parameter-free token layout and
  attention mask co-design that makes reference K/V instruction-aware while
  preserving exact reuse across denoising steps, resolving a structural
  trade-off that cannot be addressed by attention-mask design alone.

  \item We develop a two-stage architectural recovery procedure that combines
  teacher-forced velocity distillation with a short on-policy stage at
  student-visited states. To our knowledge, this is the first application of
  on-policy distillation to architectural recovery in diffusion models.

  \item We demonstrate AnchorCache across image, speech, and video generation.
  It matches full-attention quality across task-specific benchmarks, while its
  efficiency gains increase with the reference context and DiT inference
  speedups reach $6.40\times$.
\end{itemize}

\section{Related Work}
\label{sec:related}

\paragraph{Efficient diffusion transformers.}
Prior work accelerates diffusion transformers along two broad
directions: reducing per-step computation and reusing intermediate
states across denoising steps; a complementary line of work reduces the
number of sampling steps through distillation \citep{song2023consistency}.
Sparse-attention methods lower per-step cost by evaluating a subset of
query--key interactions \citep{zhang2025spargeattention,
xi2025sparsevideogen,li2025radialattention,wu2025vmoba,
zhang2026spargeattention2}; however, their dynamic and irregular
sparsity complicates kernel implementation, so reductions in attention
FLOPs do not always translate into commensurate wall-clock speedups.

Along the reuse direction, temporal caching shares activations across
nearby denoising steps \citep{selvaraju2024fora,liu2024smoothcache,
liu2025teacache,ma2024l2c,cui2026bwcache}, at the cost of approximation
error. Structural methods instead isolate fixed condition tokens from
the denoising target, enabling exact K and V reuse
\citep{zhang2025easycontrol,tan2025ominicontrol2}, but they require
additional conditioning modules and leave the cached states unaware of
the instruction. 

\paragraph{Recovery after architectural changes.}
Modifying a pretrained diffusion architecture often causes an initial quality
drop. Lightweight attention changes can often be recovered through short
fine-tuning or distillation \citep{zhang2026spargeattention2}. More substantial
changes replace softmax attention with linear or convolutional operators and
typically require longer retraining
\citep{liu2024clear,dong2025convfusion,becker2025edit,wang2025lit}.
These methods generally learn from teacher-forced examples drawn from a fixed
training distribution.

\paragraph{On-policy distillation.}
On-policy distillation trains students on states induced by their own
trajectories and has been widely studied in language models
\citep{agarwal2024gkd,gu2024minillm,ko2025distillm2,qwen2025qwen3,
lu2025onpolicydistillation}.
Related ideas have also been explored in generative vision for few-step
diffusion and autoregressive video generation
\citep{song2023consistency,yin2025causvid,huang2025selfforcing,
cui2025selfforcingpp}.
These works primarily target sampling acceleration or generative training,
rather than recovery from an attention-structure modification.

\section{Methodology}
\label{sec:method}


\subsection{Preliminaries}
\label{sec:prelim}

\paragraph{In-context diffusion transformers.}
In-context diffusion transformers concatenate instruction, target, and
reference tokens into a single sequence for joint attention
\citep{peebles2023dit,esser2024sd3,tan2025ominicontrol,
wu2025qwenimage}. Let
$\mathrm{LT}\in\mathbb{R}^{N_{\mathrm{LT}}\times d}$ denote the instruction
tokens,
$X_t\in\mathbb{R}^{N_X\times d}$ the target tokens at denoising step $t$,
and
$\mathrm{R}^{(m)}
\in\mathbb{R}^{N_{\mathrm{R}}^{(m)}\times d}$
the tokens of the $m$-th reference input, where $d$ is the hidden dimension.
For $M$ reference inputs, we concatenate their tokens as

\begin{equation}
  \mathrm{R}
  =
  [\mathrm{R}^{(1)};\ldots;\mathrm{R}^{(M)}],
  \qquad
  \mathbf{S}_t
  =
  [\mathrm{LT};X_t;\mathrm{R}],
  \label{eq:incontext-sequence}
\end{equation}

where
$N_{\mathrm{R}}=\sum_{m=1}^{M}N_{\mathrm{R}}^{(m)}$
is the total number of reference tokens.
The instruction and reference inputs remain fixed throughout denoising,
whereas $X_t$ changes at every step. Their layerwise representations,
however, may still change across denoising steps through attention to $X_t$.

\paragraph{Exact reference caching.}
Let $\mathrm{R}^{(\ell)}(t)$ denote the reference states at transformer
layer $\ell$ and denoising step $t$. Their attention keys and values are

\begin{equation}
  K_{\mathrm{R}}^{(\ell)}(t)
  =
  \mathrm{R}^{(\ell)}(t)W_K^{(\ell)},
  \qquad
  V_{\mathrm{R}}^{(\ell)}(t)
  =
  \mathrm{R}^{(\ell)}(t)W_V^{(\ell)}.
  \label{eq:reference-kv}
\end{equation}

We call the reference branch \emph{exactly cacheable} when
$K_{\mathrm{R}}^{(\ell)}(t)$ and $V_{\mathrm{R}}^{(\ell)}(t)$ remain
unchanged across denoising steps at every layer. This property holds by
construction when the reference-side inputs are fixed and no step-dependent
state, in particular $X_t$, can reach the reference branch. The reference
K/V can then be computed once and reused throughout denoising.


\begin{figure}[tbp]
  \centering
  \includegraphics[width=\linewidth]{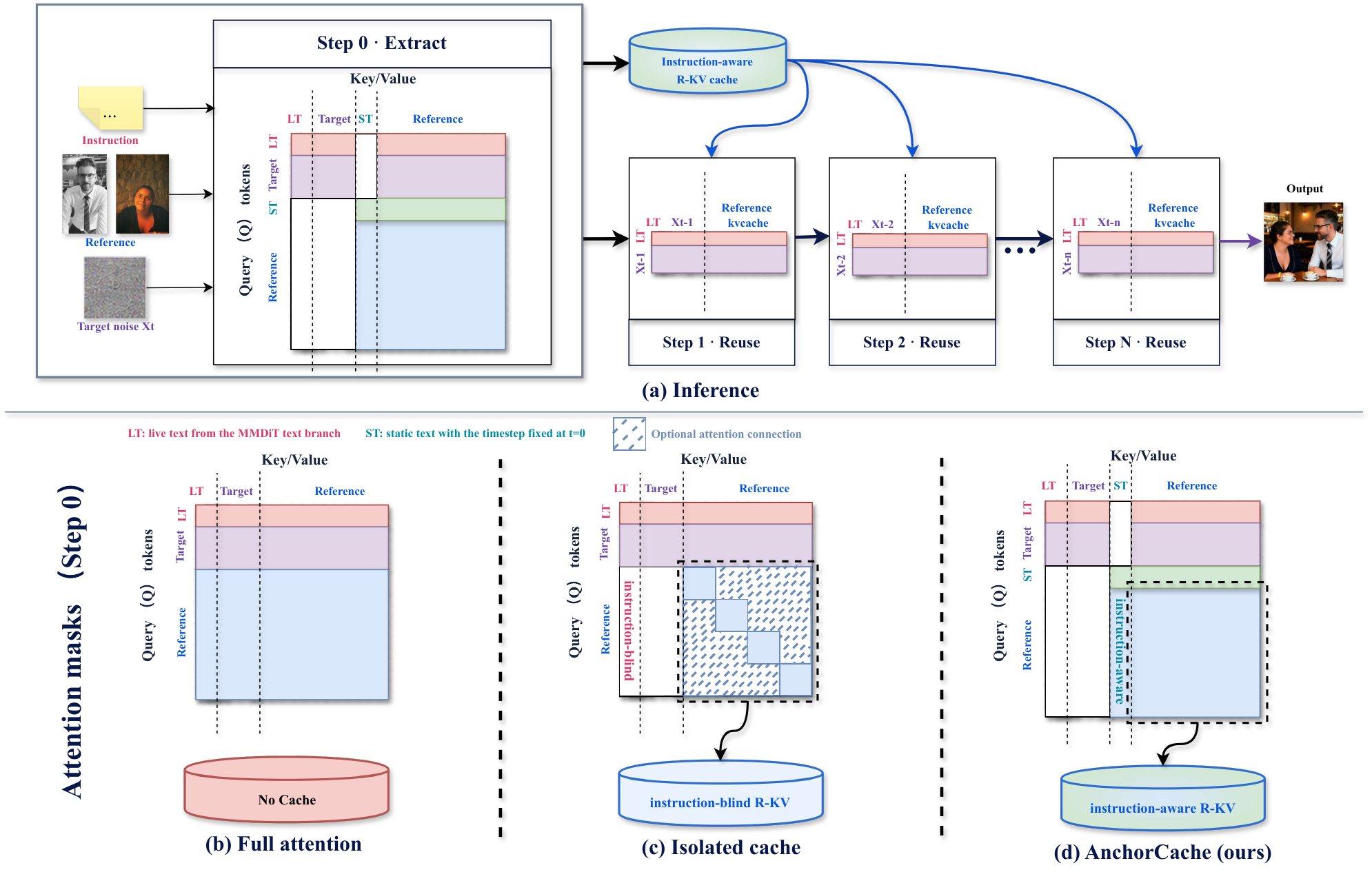}
  \caption{
  \textbf{Overview of AnchorCache.}
  Top (a): static text anchors condition the reference branch once during cache
  construction. Only the resulting reference K/V are retained and reused
  across denoising steps.
  Bottom: attention structures for (b) full attention, (c) isolated cache,
  and (d) AnchorCache, where rows denote queries and columns denote keys
  and values.
  }
  \label{fig:overview}
\end{figure}


\subsection{The Cacheability--Conditioning Trade-off}
\label{sec:constraint}

Exact cacheability does not preserve the conditioning behavior of full
attention. As illustrated by the isolated-cache attention mask in
Figure~\ref{fig:overview}(c), isolating the reference branch from $X_t$
enables exact K/V reuse, but also removes instruction conditioning from the
cached reference representations, weakening instruction--reference binding
and potentially degrading reference fidelity (Figure~\ref{fig:qualitative}).

This limitation becomes more pronounced as the reference context grows.
Without instruction-aware reference representations, the dynamic target queries
must discriminate instruction-relevant content from an increasingly large set
of competing reference tokens. A simple softmax analysis shows that maintaining
a fixed attention mass on a relevant token requires the corresponding logit
margin to grow logarithmically with the number of competing tokens; we provide
the derivation in Appendix~\ref{app:query-pressure}. Thus, when references are
instruction-agnostic, instruction-specific retrieval places a greater burden on
the dynamic query branch.

Directly allowing references to attend to the original instruction states does
not solve this problem: these states are themselves contextualized by $X_t$,
creating the indirect path
$X_t\rightarrow\mathrm{LT}\rightarrow\mathrm{R}$ and making the reference
states step-dependent. We therefore require reference representations that are
instruction-aware yet isolated from the evolving target.


\subsection{AnchorCache: Static Text Anchors}
\label{sec:anchor}

We introduce \textbf{AnchorCache}, which adds a static instruction path to
condition the reference branch without exposing it to the evolving target.

\paragraph{Static text anchors.}
We augment $[\mathrm{LT};X_t;\mathrm{R}]$ with a copy of the instruction
tokens, denoted by $\mathrm{ST}$, yielding
$[\mathrm{LT};X_t;\mathrm{ST};\mathrm{R}]$.
The anchors reuse the instruction embeddings with a fixed timestep
$t_{\mathrm{ST}}=0$. As shown in Figure~\ref{fig:overview}(d), $\mathrm{ST}$ and $\mathrm{R}$
attend to each other, while both are isolated from $\mathrm{LT}$ and $X_t$.
The dynamic branch retains the original interaction among
$\mathrm{LT}$, $X_t$, and $\mathrm{R}$, but does not attend to
$\mathrm{ST}$. This provides the path
$\mathrm{ST}\rightarrow\mathrm{R}$ without introducing
$X_t\rightarrow\mathrm{R}$. The static subgraph $[\mathrm{ST},\mathrm{R}]$ is computed once before
denoising. Only the resulting reference K/V are retained and reused across
steps; the anchor states are discarded.

\paragraph{Instruction-aware reference retrieval.}
Under an isolated cache, target--reference retrieval has the form
$q_X(X_t,c)^\top k_R(r)$, where the reference representation itself is
instruction-agnostic. AnchorCache instead yields
\begin{equation}
q_X(X_t,c)^\top k_R(r)
\quad\longrightarrow\quad
q_X(X_t,c)^\top k_R(r,c).
\label{eq:anchor-retrieval}
\end{equation}
Conditional information can therefore be represented on both sides of the
attention interaction, rather than primarily through the dynamic target query.
This provides additional capacity for resolving the reference competition
described in Section~\ref{sec:constraint} while preserving exact K/V reuse.

\paragraph{Parameter-free construction.}
AnchorCache introduces no learnable parameters and reuses the existing
instruction embeddings, attention projections, and transformer weights.
The computational overhead of introducing static text anchors is
analyzed in Appendix~\ref{app:cost}.

\begin{figure}[tbp]
  \centering
  \includegraphics[width=0.96\linewidth]{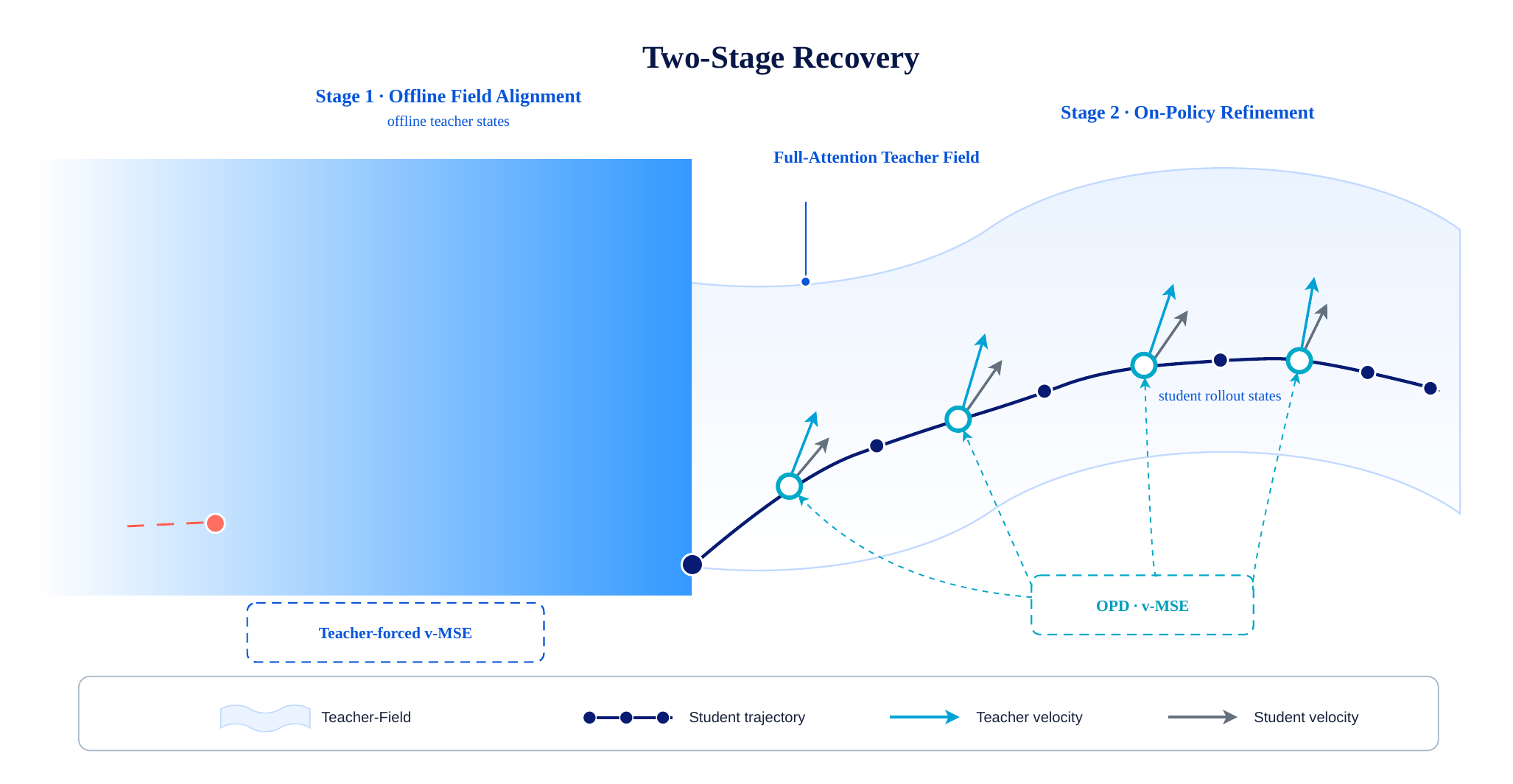}
  \caption{
  \textbf{Two-stage architectural recovery.}
  Stage~1 queries the teacher at data-derived interpolants.
  Stage~2 queries the same teacher at states visited by the student and
  corrects residual errors along the student's inference trajectory.
  }
  \label{fig:training}
\end{figure}

\subsection{Two-Stage Architectural Recovery}
\label{sec:recovery}

Although AnchorCache enables instruction-aware exact reference caching,
directly converting a pretrained full-attention model to the AnchorCache
topology initially degrades generation quality because the inherited
parameters were optimized under a different attention structure. We recover
the lost performance through the two-stage procedure illustrated in
Figure~\ref{fig:training}.

On-policy distillation reduces train--inference mismatch by supervising the
student on states induced by its own predictions
\citep{agarwal2024gkd,gu2024minillm}. We apply this principle to diffusion
trajectories after an initial teacher-forced recovery stage. Let $v_{\mathrm{full}}$ denote the frozen full-attention teacher and
$v_\theta$ the AnchorCache student.

\paragraph{Stage 1: teacher-forced velocity distillation.}
Given a clean sample $x_0$ and Gaussian noise
$x_1\sim\mathcal{N}(0,I)$, we construct the flow-matching interpolant
\citep{lipman2023flowmatching}

\begin{equation}
  x_t=(1-t)x_0+t x_1
  \label{eq:flow-interpolant}
\end{equation}

and minimize

\begin{equation}
  \mathcal{L}_{\mathrm{TF}}
  =
  \mathbb{E}_{x_0,x_1,c,t}
  \left[
    \left\|
    v_\theta(x_t,c,t)
    -
    v_{\mathrm{full}}(x_t,c,t)
    \right\|_2^2
  \right],
  \label{eq:teacher-forced}
\end{equation}

where $c$ contains the instruction and reference inputs. The data-derived
interpolants determine the query states, while the full-attention model
provides the velocity targets. This stage recovers most of the performance
lost during architectural conversion. However, it does not expose the
student to states reached by its own sampler, so residual errors can remain
along the inference trajectory after teacher-forced optimization has
largely converged.

\paragraph{Stage 2: on-policy refinement.}
We roll out the current student from
$z_1\sim\mathcal{N}(0,I)$ and collect the states $z_t^\theta$ visited along
its denoising trajectory. Let
$\bar z_t=\operatorname{sg}(z_t^\theta)$ denote the corresponding
stop-gradient states. We query the frozen full-attention teacher at these
student-visited states and optimize

\begin{equation}
  \mathcal{L}_{\mathrm{OPD}}
  =
  \mathbb{E}_{z_1,c,t}
  \left[
    \left\|
    v_\theta(\bar z_t,c,t)
    -
    v_{\mathrm{full}}(\bar z_t,c,t)
    \right\|_2^2
  \right].
  \label{eq:opd}
\end{equation}

The regression target is unchanged from Stage~1; only the query-state
distribution changes, from data-derived interpolants to states induced by the
student itself. Stop-gradient avoids backpropagation through the sampler.
This short on-policy stage directly targets the residual errors encountered
along the student's inference trajectory. Appendix~\ref{app:staged} provides
further analysis of why changing the query-state distribution can improve
recovery once teacher-forced optimization leaves a nonzero residual.

\section{Experiments}
\label{sec:exp}


\subsection{Experimental Setup}
\label{sec:exp-setup}

\paragraph{Models and tasks.}
We evaluate AnchorCache across four diffusion models spanning image, speech,
and video generation: Qwen-Image-Edit-2511 (hereafter Qwen-Image-Edit)
\citep{wu2025qwenimage},
FLUX.2 9B \citep{flux2dev2025}, OmniVoice \citep{zhu2026omnivoice},
and MiniMax-H3 \citep{minimax2026h3}.
We conduct detailed comparisons and ablation studies on Qwen-Image-Edit.
For Qwen-Image-Edit, the full-attention baseline uses the officially released
weights. The isolated-cache baseline follows the attention structure in
Figure~\ref{fig:overview}(c), while AnchorCache uses the structure in
Figure~\ref{fig:overview}(d) with static text anchors. Both cacheable variants
are initialized from the same full-attention weights.
We evaluate AnchorCache on FLUX.2 9B under a 4-step distilled sampling regime,
and further extend our evaluation to speech and video generation using
OmniVoice and MiniMax-H3, respectively.

\paragraph{Evaluation.}
For image editing, we evaluate on OmniContext \citep{wu2025omnigen2},
GEdit-Bench \citep{liu2025step1xedit}, and ImgEdit-Bench
\citep{ye2025imgedit}. OmniContext reports prompt following (PF), subject
consistency (SC), and Overall scores; GEdit-Bench reports semantic consistency
(SC), perceptual quality (PQ), and Overall scores; and ImgEdit-Bench reports
its Overall score. For zero-shot TTS, we evaluate on Seed-TTS EN and Seed-TTS ZH
\citep{anastassiou2024seedtts}, and LibriSpeech-PC
\citep{meister2023librispeechpc}, reporting word error rate (WER) and
speaker similarity (SIM-o). For OmniVideo generation, we evaluate on OpenVE-Bench \citep{he2025openve}
using Seed1.6-VL and InternVL3.5-38B as evaluators. We report Overall
scores and results for spatially aligned (SA) and non-spatially aligned
(NSA) editing categories under each evaluator.

We follow the official evaluation and aggregation protocol of each
benchmark; additional training and inference details are provided in
Appendix~\ref{app:details}.

\subsection{Main Results}
\label{sec:exp-quality}

Table~\ref{tab:crossmodel} summarizes generation quality and inference
efficiency across four models spanning image, speech, and video generation.
Table~\ref{tab:main} then provides a detailed quality--efficiency comparison
on Qwen-Image-Edit, including the full-attention model, the isolated-cache
baseline, AnchorCache, and representative diffusion acceleration methods.

%

\newcommand{\sixmetrics}[6]{%
  \begin{tabular}{@{}*{6}{>{\centering\arraybackslash}p{1.505cm}}@{}}
    #1 & #2 & #3 & #4 & #5 & #6
  \end{tabular}%
}

\newcommand{\sevenmetrics}[7]{%
  \begin{tabular}{@{}*{7}{>{\centering\arraybackslash}p{1.25cm}}@{}}
    #1 & #2 & #3 & #4 & #5 & #6 & #7
  \end{tabular}%
}

\newcommand{\spanhead}[1]{\makebox[0pt][c]{#1}}

\begin{table*}[t]
\caption{
Generation quality and inference efficiency across models and modalities.
Quality is evaluated on task-specific benchmarks and metrics.
Speedup is measured over the complete DiT denoising trajectory relative
to the corresponding full-attention model.
}
\label{tab:crossmodel}

\centering
\small
\setlength{\tabcolsep}{4.0pt}
\renewcommand{\arraystretch}{1.30}
\setlength{\arrayrulewidth}{0.4pt}

\resizebox{\textwidth}{!}{%
\begin{tabular}{l|ccc c ccc|c}

\noalign{\hrule height 0.8pt}

Model / Method
&
\multicolumn{7}{c|}{Quality}
&
Efficiency
\\

\hline


\textit{OmniImage Generation}
&
\multicolumn{7}{c|}{%
  \begin{tabular}{@{}*{7}{>{\centering\arraybackslash}p{1.25cm}}@{}}
    \multicolumn{3}{c}{\spanhead{OmniContext}}
    &
    \multicolumn{1}{c}{\spanhead{ImgEdit-Bench}}
    &
    \multicolumn{3}{c}{\spanhead{GEdit-Bench}}
    \\
    \cline{1-3}
    \cline{4-4}
    \cline{5-7}
    \spanhead{PF $\uparrow$}
    &
    \spanhead{SC $\uparrow$}
    &
    \spanhead{Overall $\uparrow$}
    &
    \spanhead{Overall $\uparrow$}
    &
    \spanhead{SC $\uparrow$}
    &
    \spanhead{PQ $\uparrow$}
    &
    \spanhead{Overall $\uparrow$}
  \end{tabular}%
}
&
Speedup $\uparrow$
\\

\hline

Qwen-Image-Edit
&
\multicolumn{7}{c|}{%
  \sevenmetrics
    {\textbf{8.550}}
    {\underline{7.867}}
    {\underline{8.119}}
    {\underline{4.324}}
    {\textbf{8.097}}
    {\textbf{7.303}}
    {\textbf{7.574}}
}
&
1.00$\times$
\\

\quad $+$ Isolated cache
&
\multicolumn{7}{c|}{%
  \sevenmetrics
    {8.155}
    {7.490}
    {7.680}
    {4.219}
    {7.911}
    {\underline{7.290}}
    {7.475}
}
&
6.45$\times$
\\

\rowcolor{taskblue}
\quad $+$ \textbf{AnchorCache}
&
\multicolumn{7}{c|}{%
  \sevenmetrics
    {\underline{8.418}}
    {\textbf{8.167}}
    {\textbf{8.185}}
    {\textbf{4.351}}
    {\underline{8.059}}
    {7.269}
    {\underline{7.540}}
}
&
6.40$\times$
\\

\hline

FLUX.2 9B \textit{(4 steps)}
&
\multicolumn{7}{c|}{%
  \sevenmetrics
    {\underline{8.602}}
    {\underline{8.043}}
    {\underline{8.230}}
    {\underline{3.919}}
    {\underline{7.650}}
    {7.376}
    {\underline{7.341}}
}
&
1.00$\times$
\\

\quad $+$ Isolated cache
&
\multicolumn{7}{c|}{%
  \sevenmetrics
    {8.545}
    {7.527}
    {7.895}
    {3.873}
    {7.331}
    {\textbf{7.452}}
    {7.132}
}
&
2.91$\times$
\\

\rowcolor{taskblue}
\quad $+$ \textbf{AnchorCache}
&
\multicolumn{7}{c|}{%
  \sevenmetrics
    {\textbf{8.665}}
    {\textbf{8.053}}
    {\textbf{8.258}}
    {\textbf{3.949}}
    {\textbf{7.785}}
    {\underline{7.431}}
    {\textbf{7.437}}
}
&
2.86$\times$
\\

\hline


\textit{Zero-shot TTS}
&
\multicolumn{7}{c|}{%
  \begin{tabular}{@{}*{6}{>{\centering\arraybackslash}p{1.505cm}}@{}}
    \multicolumn{2}{c}{\spanhead{Seed-TTS EN}}
    &
    \multicolumn{2}{c}{\spanhead{Seed-TTS ZH}}
    &
    \multicolumn{2}{c}{\spanhead{LibriSpeech-PC}}
    \\
    \cline{1-2}
    \cline{3-4}
    \cline{5-6}
    \spanhead{WER (\%) $\downarrow$}
    &
    \spanhead{SIM-o $\uparrow$}
    &
    \spanhead{WER (\%) $\downarrow$}
    &
    \spanhead{SIM-o $\uparrow$}
    &
    \spanhead{WER (\%) $\downarrow$}
    &
    \spanhead{SIM-o $\uparrow$}
  \end{tabular}%
}
&
Speedup $\uparrow$
\\

\hline

OmniVoice
&
\multicolumn{7}{c|}{%
  \sixmetrics
    {\textbf{1.62}}
    {\textbf{0.7419}}
    {\textbf{0.84}}
    {\underline{0.7779}}
    {\underline{1.41}}
    {\underline{0.7327}}
}
&
1.00$\times$
\\

\quad $+$ Isolated cache
&
\multicolumn{7}{c|}{%
  \sixmetrics
    {1.69}
    {0.6985}
    {1.07}
    {0.7687}
    {1.64}
    {0.6960}
}
&
1.50$\times$
\\

\rowcolor{taskblue}
\quad $+$ \textbf{AnchorCache}
&
\multicolumn{7}{c|}{%
  \sixmetrics
    {\underline{1.65}}
    {\underline{0.7398}}
    {\underline{0.97}}
    {\textbf{0.7780}}
    {\textbf{1.38}}
    {\textbf{0.7327}}
}
&
1.50$\times$
\\

\hline


\textit{OmniVideo Generation}
&
\multicolumn{7}{c|}{%
  \begin{tabular}{@{}*{6}{>{\centering\arraybackslash}p{1.505cm}}@{}}
    \multicolumn{3}{c}{\spanhead{OpenVE on Seed1.6-VL}}
    &
    \multicolumn{3}{c}{\spanhead{OpenVE on InternVL3.5}}
    \\
    \cline{1-3}
    \cline{4-6}
    \spanhead{Overall $\uparrow$}
    &
    \spanhead{SA $\uparrow$}
    &
    \spanhead{NSA $\uparrow$}
    &
    \spanhead{Overall $\uparrow$}
    &
    \spanhead{SA $\uparrow$}
    &
    \spanhead{NSA $\uparrow$}
  \end{tabular}%
}
&
Speedup $\uparrow$
\\

\hline

MiniMax-H3
&
\multicolumn{7}{c|}{%
  \sixmetrics
    {\underline{4.411}}
    {\underline{4.388}}
    {4.525}
    {4.125}
    {4.134}
    {4.082}
}
&
1.00$\times$
\\

\quad $+$ Isolated cache
&
\multicolumn{7}{c|}{%
  \sixmetrics
    {4.398}
    {4.351}
    {\underline{4.628}}
    {\underline{4.153}}
    {\underline{4.139}}
    {\textbf{4.224}}
}
&
1.81$\times$
\\

\rowcolor{taskblue}
\quad $+$ \textbf{AnchorCache}
&
\multicolumn{7}{c|}{%
  \sixmetrics
    {\textbf{4.479}}
    {\textbf{4.448}}
    {\textbf{4.630}}
    {\textbf{4.160}}
    {\textbf{4.155}}
    {\underline{4.187}}
}
&
1.80$\times$
\\

\noalign{\hrule height 0.8pt}

\end{tabular}%
}

\end{table*}

\noindent\textbf{Results across models and modalities.}
Table~\ref{tab:crossmodel} covers substantially different generative settings:
Qwen-Image-Edit uses multi-step, multi-reference image editing; FLUX.2 9B
uses an aggressively distilled four-step sampler; and we apply the same construction to speech and video with OmniVoice and MiniMax-H3. Across these settings,
AnchorCache preserves the quality of the corresponding full-attention model.
For example, it obtains an OmniContext Overall score of $8.258$ versus $8.230$
for FLUX.2 9B, and remains close to the full-attention model on all three speech
evaluations. On MiniMax-H3, AnchorCache matches the full-attention teacher's
generation quality on OpenVE-Bench, achieving Overall scores of
$4.479$ versus $4.411$ under Seed1.6-VL and $4.160$ versus $4.125$
under InternVL3.5-38B. These results demonstrate that the same cache
construction generalizes across architectures, modalities, and
sampling regimes.

\begin{table*}[tbp]
  \caption{
  \textbf{Generation quality on OmniContext and inference efficiency.}
  \emph{All 400} evaluates the complete benchmark.
  Speedup is measured over the complete 40-step generation process
  with 8 reference images and is relative to the full-attention
  base model.
  }
  \label{tab:main}
  \centering
  \scriptsize
  \setlength{\tabcolsep}{3.4pt}
  \renewcommand{\arraystretch}{1.20}

  \resizebox{\textwidth}{!}{%
  \begin{tabular}{@{}l cc cc ccc c@{}}
    \toprule
    & \multicolumn{2}{c}{\textsc{Multiple}}
    & \multicolumn{2}{c}{\textsc{Scene}}
    & \multicolumn{3}{c}{\emph{All 400}}
    & \multicolumn{1}{c}{Efficiency} \\
    \cmidrule(lr){2-3}
    \cmidrule(lr){4-5}
    \cmidrule(lr){6-8}
    \cmidrule(lr){9-9}

    Method
      & PF $\uparrow$ & SC $\uparrow$
      & PF $\uparrow$ & SC $\uparrow$
      & PF $\uparrow$ & SC $\uparrow$ & Overall $\uparrow$
      & Speedup $\uparrow$ \\
    \midrule

    Base (Qwen-Image-Edit)
      & \underline{8.660} & 8.020
      & \textbf{8.180} & \underline{7.693}
      & \textbf{8.550} & 7.867 & \underline{8.119}
      & 1.00$\times$ \\

    \addlinespace[2pt]
    SVDQuant \citep{li2025svdquant}
      & 7.913 & 7.807
      & 7.247 & 6.967
      & 7.888 & 7.520 & 7.582
      & 1.49$\times$ \\

    TeaCache \citep{liu2025teacache}
      & 8.513 & 7.873
      & 7.813 & 7.427
      & 8.330 & 7.675 & 7.896
      & 3.56$\times$ \\

    SageAttention2 \citep{zhang2025sageattention2}
      & \textbf{8.720} & \underline{8.280}
      & \underline{7.993} & 7.580
      & \underline{8.495} & \underline{7.923} & 8.105
      & 1.47$\times$ \\

    ToPi \citep{lin2026topi}
      & 8.580 & 8.107
      & 7.940 & 7.447
      & 8.400 & 7.692 & 7.891
      & 1.42$\times$ \\

    \midrule

    Isolated cache
      & 8.053 & 7.967
      & 7.720 & 7.287
      & 8.155 & 7.490 & 7.680
      & 6.45$\times$ \\

    \textbf{AnchorCache (ours)}
      & 8.573 & \textbf{8.453}
      & 7.913 & \textbf{7.947}
      & 8.418 & \textbf{8.167} & \textbf{8.185}
      & 6.40$\times$ \\

    \bottomrule
  \end{tabular}%
  }
\end{table*}

\noindent\textbf{Quality and efficiency.}
AnchorCache substantially reduces end-to-end DiT inference latency while preserving full-attention-level generation quality.
 Under the setting in
Table~\ref{tab:main}, it achieves a $6.40\times$ speedup, nearly matching the
$6.45\times$ isolated cache, which shows that static text anchors add little
latency beyond cache construction. At this similar latency, AnchorCache raises
the OmniContext Overall score from $7.680$ for the isolated cache to $8.185$,
compared with $8.119$ for full attention; the largest gain is in subject
consistency, in line with Figure~\ref{fig:qualitative}. Representative
acceleration methods exhibit different quality--efficiency trade-offs and
operate on complementary aspects of diffusion inference rather than serving
as direct architectural counterparts.

\noindent\textbf{Complementarity with diffusion acceleration methods.}
AnchorCache is complementary to commonly used diffusion acceleration
techniques, such as SageAttention2. Combining the two yields a
quality--efficiency trend similar to that observed when applying
SageAttention2 to the full-attention baseline, supporting their
complementarity. Detailed quantitative results are provided in the
appendix~\ref{app:complementary-acceleration}.


\subsection{Ablation Study}
\label{sec:exp-ablation}


\noindent\textbf{Effect of static text anchors and recovery stages.}
Table~\ref{tab:robust} separates the effects of the attention topology and the
two recovery stages. At the same recovery stage, AnchorCache consistently
outperforms the isolated cache across all three benchmarks. On OmniContext,
for example, static anchors improve Stage~1 Overall from $7.529$ to $8.019$;
after Stage~2, AnchorCache reaches $8.185$ versus $7.680$ for the isolated
cache. Stage~2 also improves AnchorCache on ImgEdit-Bench and GEdit-Bench,
showing that instruction-aware static anchors and on-policy refinement provide
complementary gains.

\begin{table*}[tbp]
  \caption{
  \textbf{Ablation of static text anchors and recovery stages.}
  Results are reported across three image-editing benchmarks for the
  full-attention model and both cacheable architectures after Stage~1 and
  Stage~2 recovery.
  }
  \label{tab:robust}
  \centering
  \small
  \setlength{\tabcolsep}{4.5pt}
  \renewcommand{\arraystretch}{1.18}
  \begin{tabular}{@{}l ccc c ccc@{}}
    \toprule
    & \multicolumn{3}{c}{OmniContext}
    & \multicolumn{1}{c}{ImgEdit-Bench}
    & \multicolumn{3}{c}{GEdit-Bench} \\
    \cmidrule(lr){2-4}
    \cmidrule(lr){5-5}
    \cmidrule(lr){6-8}
    Method
      & PF $\uparrow$
      & SC $\uparrow$
      & Overall $\uparrow$
      & Overall $\uparrow$
      & SC $\uparrow$
      & PQ $\uparrow$
      & Overall $\uparrow$ \\
    \midrule

    Base (Qwen-Image-Edit)
      & \textbf{8.550}
      & 7.867
      & \underline{8.119}
      & \underline{4.324}
      & \textbf{8.097}
      & \textbf{7.303}
      & \textbf{7.574} \\

    \addlinespace[2pt]
    Isolated, Stage~1
      & 8.180
      & 7.203
      & 7.529
      & 4.184
      & 7.929
      & 7.181
      & 7.440 \\

    Isolated, $+$ Stage~2
      & 8.155
      & 7.490
      & 7.680
      & 4.219
      & 7.911
      & \underline{7.290}
      & 7.475 \\

    \addlinespace[2pt]
    AnchorCache, Stage~1
      & 8.285
      & \underline{7.970}
      & 8.019
      & 4.293
      & 7.940
      & 7.252
      & 7.460 \\

    AnchorCache, $+$ Stage~2
      & \underline{8.418}
      & \textbf{8.167}
      & \textbf{8.185}
      & \textbf{4.351}
      & \underline{8.059}
      & 7.269
      & \underline{7.540} \\

    \bottomrule
  \end{tabular}
\end{table*}


\noindent\textbf{Teacher-forced saturation and on-policy recovery.}
\label{sec:exp-opd}
We next test whether the Stage~2 gain can be obtained by simply continuing
teacher-forced distillation. Starting from the same $30$k Stage~1 checkpoint,
continued teacher-forced training remains nearly saturated across the
30k--33k checkpoints, with an Overall mean of $8.015\pm0.015$. In contrast,
on-policy refinement reaches $8.057$ after only $250$ updates and $8.185$
after $500$, with every evaluated on-policy checkpoint exceeding the best
continued teacher-forced checkpoint
(Figure~\ref{fig:recovery_curves}). The remaining error is therefore better
addressed by supervising the model at student-visited states than by further
optimization on data-derived interpolants.

\begin{figure}[tbp]
    \centering
    \includegraphics[width=0.98\linewidth]
    {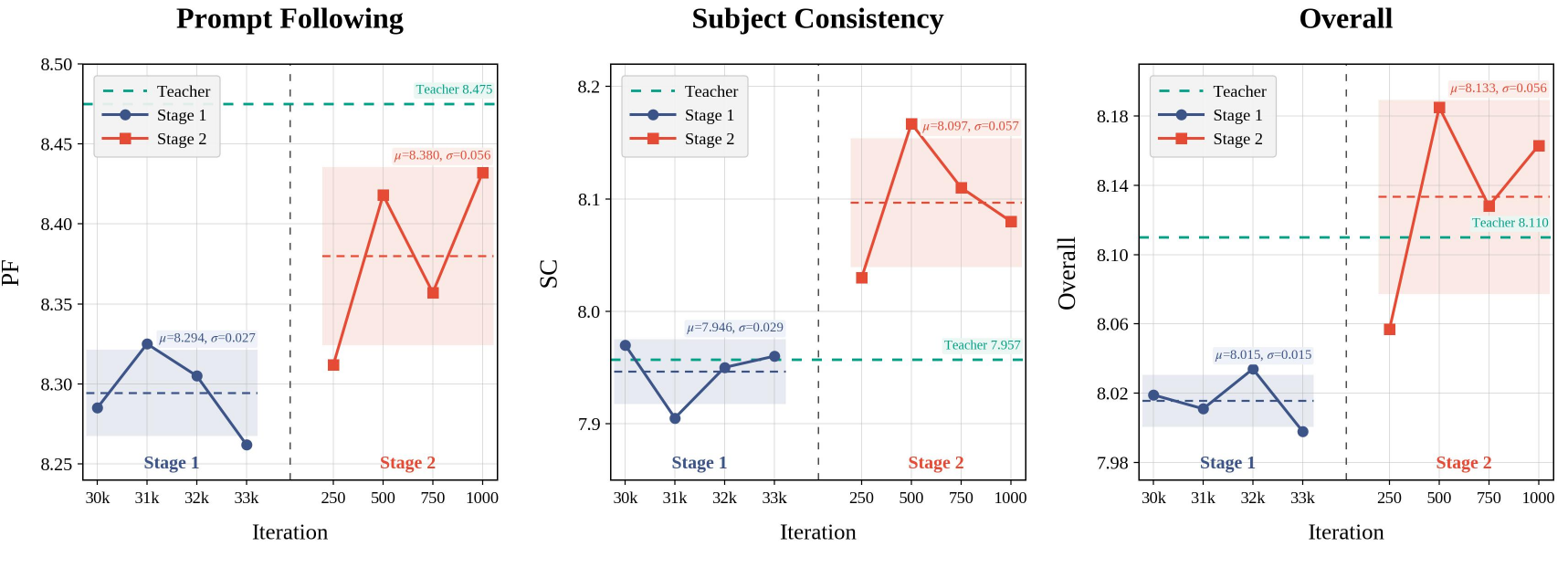}
    \caption{
    \textbf{Teacher-forced saturation and on-policy recovery.}
    Starting from the same 30k Stage~1 checkpoint, we compare continued
    teacher-forced training with Stage~2 on-policy recovery.
    Shaded regions denote mean $\pm$ standard deviation, and the green
    dashed line indicates the full-attention teacher.
    }
    \label{fig:recovery_curves}
\end{figure}


\begin{figure*}[t]
    \centering

    \begin{minipage}[t]{0.49\textwidth}
        \centering
        \captionsetup{type=table}
        \caption{
        \textbf{OPD query design.}
        All configurations are initialized from the same 30k Stage~1 checkpoint
        and evaluated after 500 on-policy updates.
        Overall scores are reported for different query intervals and numbers of
        sampled states $N_q$, with $\Delta$ measured from the Stage~1 starting
        point ($8.019$).
        }
        \label{tab:opd-query}

        \vspace{2pt}

        \scriptsize
        \setlength{\tabcolsep}{3.5pt}
        \renewcommand{\arraystretch}{1.10}
        \begin{tabular}{@{}lcccc@{}}
            \toprule
            & \multicolumn{2}{c}{$N_q=1$}
            & \multicolumn{2}{c}{$N_q=4$} \\
            \cmidrule(lr){2-3}
            \cmidrule(lr){4-5}
            Query interval
              & Overall & $\Delta$
              & Overall & $\Delta$ \\
            \midrule
            Full trajectory
              & 8.084 & $+0.065$
              & \textbf{8.185} & $\mathbf{+0.166}$ \\
            High-noise
              & 8.060 & $+0.041$
              & 8.131 & $+0.112$ \\
            Mid-noise
              & 8.073 & $+0.054$
              & 8.051 & $+0.032$ \\
            Low-noise
              & 7.936 & $-0.083$
              & 7.888 & $-0.131$ \\
            \bottomrule
        \end{tabular}
    \end{minipage}
    \hfill
    \begin{minipage}[t]{0.49\textwidth}
        \centering
        \captionsetup{type=figure}

        \includegraphics[width=\linewidth]
        {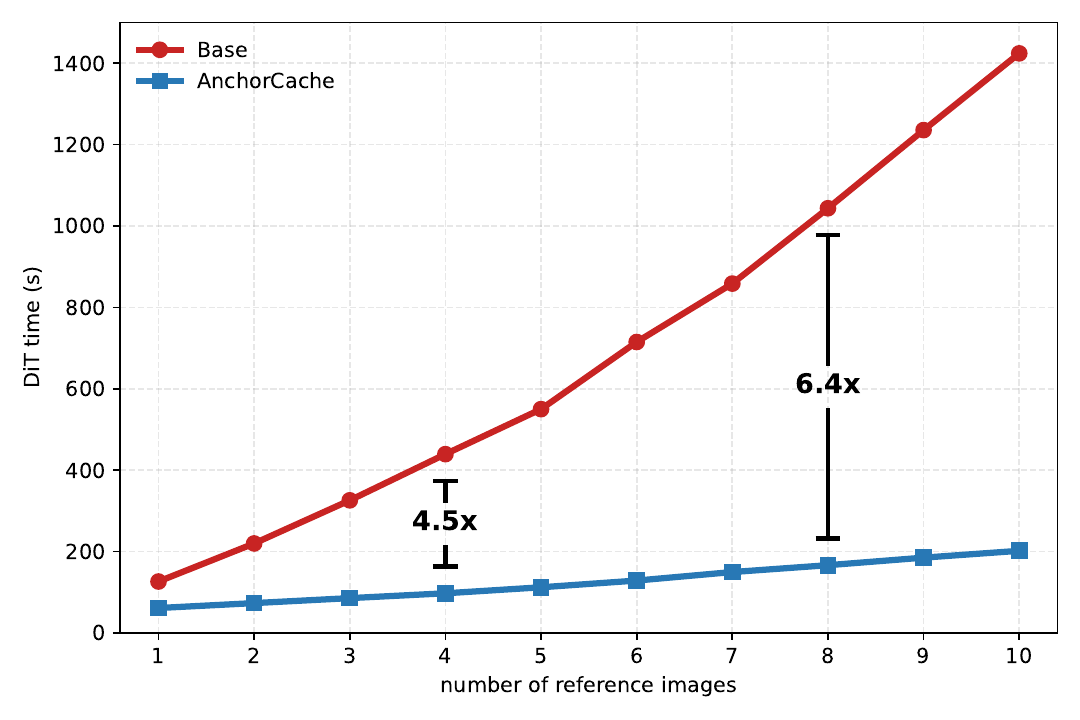}

        \caption{
        \textbf{Scaling with the number of references.}
        DiT inference latency as the number of reference images increases.
        }
        \label{fig:reference-scaling}
    \end{minipage}

\end{figure*}

\noindent\textbf{OPD query design.}
We vary both the query interval and the number of states sampled from each
student rollout. Increasing the number of sampled states does not uniformly
improve recovery: its benefit depends strongly on where the queries are drawn
from along the trajectory. In particular, concentrating queries in a restricted
noise range can even degrade performance, whereas sampling across the full
trajectory consistently performs best. We therefore adopt full-trajectory
sampling with $N_q=4$ as our default OPD configuration; 
further experimental details are provided in Appendix~\ref{app:opd-query}.
\medskip

\noindent\textbf{Scaling with the number of references.}
We further evaluate how DiT denoising latency scales with the number of
reference images using Qwen-Image-Edit (20B). Under full attention, every
additional reference introduces tokens that must be repeatedly processed
throughout the denoising trajectory, causing latency to grow rapidly with
the reference context. In contrast, AnchorCache computes the reference
representations once and reuses the cached K/V at subsequent denoising
steps. Consequently, its relative efficiency gain increases with the number
of references, as shown in Figure~\ref{fig:reference-scaling}.

\section{Conclusion}
\label{sec:conclusion}

In this work, we studied efficient inference for in-context diffusion models.
Exact reference caching requires the reference representations to remain
independent of the denoising target, but isolating the reference branch also
removes instruction conditioning. We introduced \textbf{AnchorCache}, a
parameter-free token layout and attention mask co-design that uses static text
anchors to condition the reference representations during cache construction
while preserving exact K/V reuse. To restore generation quality after this architectural conversion, we developed a two-stage procedure combining teacher-forced velocity distillation with short on-policy refinement at student-visited states. Across image, speech, and video generation,
AnchorCache preserves the quality of the corresponding full-attention models
while avoiding repeated reference-side computation across denoising steps.
Its efficiency advantage further increases as the reference context grows.
Overall, our results show that exact reference caching need not come at the
cost of instruction-aware conditioning, providing a practical approach to
more efficient diffusion generation.

\clearpage
\bibliography{refs}

@inproceedings{zhang2025spargeattention,
  title     = {{SpargeAttention}: Accurate and Training-Free Sparse Attention Accelerating Any Model Inference},
  author    = {Zhang, Jintao and Xiang, Chendong and Huang, Haofeng and Wei, Jia and Xi, Haocheng and Zhu, Jun and Chen, Jianfei},
  booktitle = {International Conference on Machine Learning},
  year      = {2025},
  eprint    = {2502.18137},
  archiveprefix = {arXiv}
}

@misc{xi2025sparsevideogen,
  title         = {Sparse VideoGen: Accelerating Video Diffusion Transformers with Spatial-Temporal Sparsity},
  author        = {Xi, Haocheng and Yang, Shuo and Zhao, Yilong and Xu, Chenfeng and Li, Muyang and Li, Xiuyu and Lin, Yujun and Cai, Han and Zhang, Jintao and Li, Dacheng and Chen, Jianfei and Stoica, Ion and Keutzer, Kurt and Han, Song},
  year          = {2025},
  eprint        = {2502.01776},
  archiveprefix = {arXiv},
  primaryclass  = {cs.CV}
}

@inproceedings{li2025radialattention,
  title     = {Radial Attention: {$O(n\log n)$} Sparse Attention with Energy Decay for Long Video Generation},
  author    = {Li, Xingyang and Li, Muyang and Cai, Tianle and Xi, Haocheng and Yang, Shuo and Lin, Yujun and Zhang, Lvmin and Yang, Songlin and Hu, Jinbo and Peng, Kelly and Agrawala, Maneesh and Stoica, Ion and Keutzer, Kurt and Han, Song},
  booktitle = {Advances in Neural Information Processing Systems},
  year      = {2025},
  eprint    = {2506.19852},
  archiveprefix = {arXiv}
}

@misc{wu2025vmoba,
  title         = {{VMoBA}: Mixture-of-Block Attention for Video Diffusion Models},
  author        = {Wu, Jianzong and Hou, Liang and Yang, Haotian and Tao, Xin and Tian, Ye and Wan, Pengfei and Zhang, Di and Tong, Yunhai},
  year          = {2025},
  eprint        = {2506.23858},
  archiveprefix = {arXiv},
  primaryclass  = {cs.CV}
}

@misc{zhang2026spargeattention2,
  title         = {{SpargeAttention2}: Trainable Sparse Attention via Hybrid Top-k+Top-p Masking and Distillation Fine-Tuning},
  author        = {Zhang, Jintao and Jiang, Kai and Xiang, Chendong and Feng, Weiqi and Hu, Yuezhou and Xi, Haocheng and Chen, Jianfei and Zhu, Jun},
  year          = {2026},
  eprint        = {2602.13515},
  archiveprefix = {arXiv},
  primaryclass  = {cs.CV}
}

@misc{selvaraju2024fora,
  title         = {{FORA}: Fast-Forward Caching in Diffusion Transformer Acceleration},
  author        = {Selvaraju, Pratheba and Ding, Tianyu and Chen, Tianyi and Zharkov, Ilya and Liang, Luming},
  year          = {2024},
  eprint        = {2407.01425},
  archiveprefix = {arXiv},
  primaryclass  = {cs.CV}
}

@misc{liu2024smoothcache,
  title         = {{SmoothCache}: A Universal Inference Acceleration Technique for Diffusion Transformers},
  author        = {Liu, Joseph and Geddes, Joshua and Guo, Ziyu and Jiang, Haomiao and Nandwana, Mahesh Kumar},
  year          = {2024},
  eprint        = {2411.10510},
  archiveprefix = {arXiv},
  primaryclass  = {cs.LG}
}

@inproceedings{ma2024l2c,
  title     = {Learning-to-Cache: Accelerating Diffusion Transformer via Layer Caching},
  author    = {Ma, Xinyin and Fang, Gongfan and Mi, Michael Bi and Wang, Xinchao},
  booktitle = {Advances in Neural Information Processing Systems},
  year      = {2024},
  eprint    = {2406.01733},
  archiveprefix = {arXiv}
}

@misc{cui2026bwcache,
  title         = {{BWCache}: Accelerating Video Diffusion Transformers through Block-Wise Caching},
  author        = {Cui, Hanshuai and Tang, Zhiqing and Xu, Zhifei and Yao, Zhi and Zeng, Wenyi and Jia, Weijia},
  year          = {2026},
  eprint        = {2509.13789},
  archiveprefix = {arXiv},
  primaryclass  = {cs.CV}
}

@misc{zhang2025easycontrol,
  title         = {{EasyControl}: Adding Efficient and Flexible Control for Diffusion Transformer},
  author        = {Zhang, Yuxuan and Yuan, Yirui and Song, Yiren and Wang, Haofan and Liu, Jiaming},
  year          = {2025},
  eprint        = {2503.07027},
  archiveprefix = {arXiv},
  primaryclass  = {cs.CV}
}

@misc{tan2025ominicontrol2,
  title         = {{OminiControl2}: Efficient Conditioning for Diffusion Transformers},
  author        = {Tan, Zhenxiong and Xue, Qiaochu and Yang, Xingyi and Liu, Songhua and Wang, Xinchao},
  year          = {2025},
  eprint        = {2503.08280},
  archiveprefix = {arXiv},
  primaryclass  = {cs.CV}
}

@inproceedings{tan2025ominicontrol,
  title     = {{OminiControl}: Minimal and Universal Control for Diffusion Transformer},
  author    = {Tan, Zhenxiong and Liu, Songhua and Yang, Xingyi and Xue, Qiaochu and Wang, Xinchao},
  booktitle = {Proceedings of the IEEE/CVF International Conference on Computer Vision},
  year      = {2025},
  eprint    = {2411.15098},
  archiveprefix = {arXiv}
}

@misc{wu2025qwenimage,
  title         = {{Qwen-Image} Technical Report},
  author        = {Wu, Chenfei and Li, Jiahao and Zhou, Jingren and Lin, Junyang and Gao, Kaiyuan and Yan, Kun and Yin, Sheng-ming and Bai, Shuai and Xu, Xiao and Chen, Yilei and Chen, Yuxiang and Tang, Zecheng and Zhang, Zekai and Wang, Zhengyi and Yang, An and Yu, Bowen and Cheng, Chen and Liu, Dayiheng and Li, Deqing and Zhang, Hang and Meng, Hao and Wei, Hu and Ni, Jingyuan and Chen, Kai and Cao, Kuan and Peng, Liang and Qu, Lin and Wu, Minggang and Wang, Peng and Yu, Shuting and Wen, Tingkun and Feng, Wensen and Xu, Xiaoxiao and Wang, Yi and Zhang, Yichang and Zhu, Yongqiang and Wu, Yujia and Cai, Yuxuan and Liu, Zenan},
  year          = {2025},
  eprint        = {2508.02324},
  archiveprefix = {arXiv},
  primaryclass  = {cs.CV}
}

@inproceedings{peebles2023dit,
  title     = {Scalable Diffusion Models with Transformers},
  author    = {Peebles, William and Xie, Saining},
  booktitle = {Proceedings of the IEEE/CVF International Conference on Computer Vision},
  year      = {2023},
  eprint    = {2212.09748},
  archiveprefix = {arXiv}
}

@inproceedings{esser2024sd3,
  title     = {Scaling Rectified Flow Transformers for High-Resolution Image Synthesis},
  author    = {Esser, Patrick and Kulal, Sumith and Blattmann, Andreas and Entezari, Rahim and M{\"u}ller, Jonas and Saini, Harry and Levi, Yam and Lorenz, Dominik and Sauer, Axel and Boesel, Frederic and Podell, Dustin and Dockhorn, Tim and English, Zion and Lacey, Kyle and Goodwin, Alex and Marek, Yannik and Rombach, Robin},
  booktitle = {International Conference on Machine Learning},
  year      = {2024},
  eprint    = {2403.03206},
  archiveprefix = {arXiv}
}

@misc{liu2024clear,
  title         = {{CLEAR}: Conv-Like Linearization Revs Pre-Trained Diffusion Transformers Up},
  author        = {Liu, Songhua and Tan, Zhenxiong and Wang, Xinchao},
  year          = {2024},
  eprint        = {2412.16112},
  archiveprefix = {arXiv},
  primaryclass  = {cs.CV}
}

@misc{dong2025convfusion,
  title         = {Can We Achieve Efficient Diffusion without Self-Attention? Distilling Self-Attention into Convolutions},
  author        = {Dong, ZiYi and Zhou, Chengxing and Deng, Weijian and Wei, Pengxu and Ji, Xiangyang and Lin, Liang},
  year          = {2025},
  eprint        = {2504.21292},
  archiveprefix = {arXiv},
  primaryclass  = {cs.CV}
}

@misc{becker2025edit,
  title         = {{EDiT}: Efficient Diffusion Transformers with Linear Compressed Attention},
  author        = {Becker, Philipp and Mehrotra, Abhinav and Chavhan, Ruchika and Chadwick, Malcolm and Morreale, Luca and Noroozi, Mehdi and Ramos, Alberto Gil and Bhattacharya, Sourav},
  year          = {2025},
  eprint        = {2503.16726},
  archiveprefix = {arXiv},
  primaryclass  = {cs.CV}
}

@misc{wang2025lit,
  title         = {{LiT}: Delving into a Simple Linear Diffusion Transformer for Image Generation},
  author        = {Wang, Jiahao and Kang, Ning and Yao, Lewei and Chen, Mengzhao and Wu, Chengyue and Zhang, Songyang and Xue, Shuchen and Liu, Yong and Wu, Taiqiang and Liu, Xihui and Zhang, Kaipeng and Zhang, Shifeng and Shao, Wenqi and Li, Zhenguo and Luo, Ping},
  year          = {2025},
  eprint        = {2501.12976},
  archiveprefix = {arXiv},
  primaryclass  = {cs.CV}
}

@inproceedings{agarwal2024gkd,
  title     = {On-Policy Distillation of Language Models: Learning from Self-Generated Mistakes},
  author    = {Agarwal, Rishabh and Vieillard, Nino and Zhou, Yongchao and Stanczyk, Piotr and Ramos, Sabela and Geist, Matthieu and Bachem, Olivier},
  booktitle = {OpenReview},
  year      = {2024},
  eprint    = {2306.13649},
  archiveprefix = {arXiv}
}

@inproceedings{gu2024minillm,
  title     = {{MiniLLM}: On-Policy Distillation of Large Language Models},
  author    = {Gu, Yuxian and Dong, Li and Wei, Furu and Huang, Minlie},
  booktitle = {OpenReview},
  year      = {2024},
  eprint    = {2306.08543},
  archiveprefix = {arXiv}
}

@inproceedings{song2023consistency,
  title     = {Consistency Models},
  author    = {Song, Yang and Dhariwal, Prafulla and Chen, Mark and Sutskever, Ilya},
  booktitle = {International Conference on Machine Learning},
  year      = {2023},
  eprint    = {2303.01469},
  archiveprefix = {arXiv}
}

@inproceedings{huang2025selfforcing,
  title     = {Self Forcing: Bridging the Train-Test Gap in Autoregressive Video Diffusion},
  author    = {Huang, Xun and Li, Zhengqi and He, Guande and Zhou, Mingyuan and Shechtman, Eli},
  booktitle = {Advances in Neural Information Processing Systems},
  year      = {2025},
  eprint    = {2506.08009},
  archiveprefix = {arXiv}
}

@misc{liu2025step1xedit,
  title         = {{Step1X-Edit}: A Practical Framework for General Image Editing},
  author        = {Liu, Shiyu and Han, Yucheng and Xing, Peng and Yin, Fukun and Wang, Rui and Cheng, Wei and Liao, Jiaqi and Wang, Yingming and Fu, Honghao and Han, Chunrui and Li, Guopeng and Peng, Yuang and Sun, Quan and Wu, Jingwei and Cai, Yan and Ge, Zheng and Ming, Ranchen and Xia, Lei and Zeng, Xianfang and Zhu, Yibo and Jiao, Binxing and Zhang, Xiangyu and Yu, Gang and Jiang, Daxin},
  year          = {2025},
  eprint        = {2504.17761},
  archiveprefix = {arXiv},
  primaryclass  = {cs.CV}
}

@misc{ye2025imgedit,
  title         = {{ImgEdit}: A Unified Image Editing Dataset and Benchmark},
  author        = {Ye, Yang and He, Xianyi and Li, Zongjian and Lin, Bin and Yuan, Shenghai and Yan, Zhiyuan and Hou, Bohan and Yuan, Li},
  year          = {2025},
  eprint        = {2505.20275},
  archiveprefix = {arXiv},
  primaryclass  = {cs.CV}
}

@inproceedings{wei2025omniedit,
  title     = {{OmniEdit}: Building Image Editing Generalist Models Through Specialist Supervision},
  author    = {Wei, Cong and Xiong, Zheyang and Ren, Weiming and Du, Xinrun and Zhang, Ge and Chen, Wenhu},
  booktitle = {OpenReview},
  year      = {2025},
  eprint    = {2411.07199},
  archiveprefix = {arXiv}
}

@inproceedings{jiang2026orionedit,
  title     = {{OrionEdit}: Bridging Reference and Source Images for Generalized Cross-Image Editing},
  author    = {Jiang, Zeyu and Po, Lai Man and Xu, Xuyuan and Wang, Yexin and Gong, Guoping and Wu, Haoxuan and Yan, Chenbo and Li, Kun and Liu, Yuyang},
  booktitle = {Proceedings of the IEEE/CVF Conference on Computer Vision and Pattern Recognition},
  pages     = {9127--9138},
  year      = {2026}
}

@article{wu2025omnigen2,
  title={OmniGen2: Exploration to Advanced Multimodal Generation},
  author={Chenyuan Wu and Pengfei Zheng and Ruiran Yan and Shitao Xiao and Xin Luo and Yueze Wang and Wanli Li and Xiyan Jiang and Yexin Liu and Junjie Zhou and Ze Liu and Ziyi Xia and Chaofan Li and Haoge Deng and Jiahao Wang and Kun Luo and Bo Zhang and Defu Lian and Xinlong Wang and Zhongyuan Wang and Tiejun Huang and Zheng Liu},
  journal={arXiv preprint arXiv:2506.18871},
  year={2025}
}

@inproceedings{ko2025distillm2,
  title     = {{DistiLLM-2}: A Contrastive Approach Boosts the Distillation of {LLMs}},
  author    = {Ko, Jongwoo and Chen, Tianyi and Kim, Sungnyun and Ding, Tianyu and Liang, Luming and Zharkov, Ilya and Yun, Se-Young},
  booktitle = {International Conference on Machine Learning},
  year      = {2025},
  eprint    = {2503.07067},
  archiveprefix = {arXiv}
}

@misc{qwen2025qwen3,
  title         = {{Qwen3} Technical Report},
  author        = {Yang, An and others},
  year          = {2025},
  eprint        = {2505.09388},
  archiveprefix = {arXiv},
  primaryclass  = {cs.CL}
}

@article{lu2025onpolicydistillation,
  author  = {Lu, Kevin},
  title   = {On-Policy Distillation},
  journal = {Thinking Machines Lab: Connectionism},
  year    = {2025},
  doi     = {10.64434/tml.20251026},
  note    = {https://thinkingmachines.ai/blog/on-policy-distillation/}
}

@inproceedings{yin2025causvid,
  title     = {From Slow Bidirectional to Fast Autoregressive Video Diffusion Models},
  author    = {Yin, Tianwei and Zhang, Qiang and Zhang, Richard and Freeman, William T. and Durand, Fredo and Shechtman, Eli and Huang, Xun},
  booktitle = {Proceedings of the IEEE/CVF Conference on Computer Vision and Pattern Recognition},
  year      = {2025},
  eprint    = {2412.07772},
  archiveprefix = {arXiv}
}

@misc{cui2025selfforcingpp,
  title         = {{Self-Forcing++}: Towards Minute-Scale High-Quality Video Generation},
  author        = {Cui, Justin and Wu, Jie and Li, Ming and Yang, Tao and Li, Xiaojie and Wang, Rui and Bai, Andrew and Ban, Yuanhao and Hsieh, Cho-Jui},
  year          = {2025},
  eprint        = {2510.02283},
  archiveprefix = {arXiv},
  primaryclass  = {cs.CV}
}

@misc{li2026rethinkingcfgopd,
  title         = {Rethinking Classifier-Free Guidance in On-Policy Diffusion Distillation},
  author        = {Li, Bingnan and Wang, Haozhe and Xiong, Haozhong and Wu, Fangtai and Yu, Jinpeng and Shi, Yang and Liu, Jiaming and Huang, Ruihua},
  year          = {2026},
  eprint        = {2607.24731},
  archiveprefix = {arXiv},
  primaryclass  = {cs.CV}
}

@inproceedings{lipman2023flowmatching,
  title     = {Flow Matching for Generative Modeling},
  author    = {Lipman, Yaron and Chen, Ricky T. Q. and Ben-Hamu, Heli and Nickel, Maximilian and Le, Matt},
  booktitle = {OpenReview},
  year      = {2023}
}

@misc{liu2025cobra,
  title         = {{Cobra}: Efficient Line Art {CO}lorization with {BRoA}der References},
  author        = {Zhuang, Junhao and Li, Lingen and Ju, Xuan and Zhang, Zhaoyang and Yuan, Chun and Shan, Ying},
  year          = {2025},
  eprint        = {2504.12240},
  archiveprefix = {arXiv},
  primaryclass  = {cs.CV}
}

@misc{firered2026imageedit,
  title         = {{FireRed-Image-Edit-1.0} Technical Report},
  author        = {{Super Intelligence Team} and Qiao, Changhao and Hui, Chao and Li, Chen and Wang, Cunzheng and Song, Dejia and Zhang, Jiale and Li, Jing and Xiang, Qiang and Wang, Runqi and Sun, Shuang and Zhu, Wei and Tang, Xu and Hu, Yao and Chen, Yibo and Huang, Yuhao and Duan, Yuxuan and Chen, Zhiyi and Guo, Ziyuan},
  year          = {2026},
  eprint        = {2602.13344},
  archiveprefix = {arXiv},
  primaryclass  = {cs.CV}
}

@misc{longcat2025image,
  title         = {{LongCat-Image} Technical Report},
  author        = {{Meituan LongCat Team} and Ma, Hanghang and Tan, Haoxian and Huang, Jiale and Wu, Junqiang and He, Jun-Yan and Gao, Lishuai and Xiao, Songlin and Wei, Xiaoming and Ma, Xiaoqi and Cai, Xunliang and Guan, Yayong and Hu, Jie},
  year          = {2025},
  eprint        = {2512.07584},
  archiveprefix = {arXiv},
  primaryclass  = {cs.CV}
}

@misc{boogu2026image,
  title         = {{Boogu-Image-0.1}: Boosting Open-Source Unified Multimodal Understanding and Generation},
  author        = {Chen, Guoxuan and Xiao, Chufeng and Yang, Haoran and Xie, Siyue and Huang, Binxiao and Zhang, Ming and Chau, Cheuk Him and Fu, Xinyu and Lian, Yingzhao and Li, Tom S. Y. and others},
  year          = {2026},
  eprint        = {2607.13125},
  archiveprefix = {arXiv},
  primaryclass  = {cs.CV}
}

@misc{joyai2026image,
  title         = {Awaking Spatial Intelligence in Unified Multimodal Understanding and Generation},
  author        = {Song, Lin and Li, Wenbo and Ma, Guoqing and Tang, Wei and Wang, Bo and Zhang, Yuan and Yang, Yijun and Xiao, Yicheng and Liu, Jianhui and Zhang, Yanbing and others},
  year          = {2026},
  eprint        = {2605.04128},
  archiveprefix = {arXiv},
  primaryclass  = {cs.CV}
}

@misc{flux2dev2025,
  title        = {{FLUX.2}: Frontier Visual Intelligence},
  author       = {{Black Forest Labs}},
  year         = {2025},
  howpublished = {Black Forest Labs technical blog},
  note         = {https://bfl.ai/blog/flux-2}
}

@inproceedings{li2025svdquant,
  title     = {{SVDQuant}: Absorbing Outliers by Low-Rank Components for 4-Bit Diffusion Models},
  author    = {Li, Muyang and Lin, Yujun and Zhang, Zhekai and Cai, Tianle and Li, Xiuyu and Guo, Junxian and Xie, Enze and Meng, Chenlin and Zhu, Jun-Yan and Han, Song},
  booktitle = {OpenReview},
  year      = {2025},
  eprint    = {2411.05007},
  archiveprefix = {arXiv}
}

@inproceedings{liu2025teacache,
  title     = {Timestep Embedding Tells: It's Time to Cache for Video Diffusion Model},
  author    = {Liu, Feng and Zhang, Shiwei and Wang, Xiaofeng and Wei, Yujie and Qiu, Haonan and Zhao, Yuzhong and Zhang, Yingya and Ye, Qixiang and Wan, Fang},
  booktitle = {Proceedings of the IEEE/CVF Conference on Computer Vision and Pattern Recognition},
  pages     = {7353--7363},
  year      = {2025},
  eprint    = {2411.19108},
  archiveprefix = {arXiv}
}

@inproceedings{zhang2025sageattention2,
  title     = {{SageAttention2}: Efficient Attention with Thorough Outlier Smoothing and Per-thread {INT4} Quantization},
  author    = {Zhang, Jintao and Huang, Haofeng and Zhang, Pengle and Wei, Jia and Zhu, Jun and Chen, Jianfei},
  booktitle = {International Conference on Machine Learning},
  pages     = {75097--75119},
  year      = {2025},
  eprint    = {2411.10958},
  archiveprefix = {arXiv}
}

@misc{lin2026topi,
  title         = {Token Pruning for In-Context Generation in Diffusion Transformers},
  author        = {Lin, Junqing and Zheng, Xingyu and Cheng, Pei and Fu, Bin and Sun, Jingwei and Sun, Guangzhong},
  year          = {2026},
  eprint        = {2602.01609},
  archiveprefix = {arXiv},
  primaryclass  = {cs.CV}
}

@misc{zhu2026omnivoice,
  title         = {{OmniVoice}: Towards Omnilingual Zero-Shot Text-to-Speech
                   with Diffusion Language Models},
  author        = {Zhu, Han and Ye, Lingxuan and Kang, Wei and Yao, Zengwei
                   and Guo, Liyong and Kuang, Fangjun and Han, Zhifeng
                   and Zhuang, Weiji and Lin, Long and Povey, Daniel},
  year          = {2026},
  eprint        = {2604.00688},
  archiveprefix = {arXiv},
  primaryclass  = {cs.CL}
}

@misc{anastassiou2024seedtts,
  title         = {{Seed-TTS}: A Family of High-Quality Versatile Speech
                   Generation Models},
  author        = {Anastassiou, Philip and Chen, Jiawei and Chen, Jitong
                   and Chen, Yuanzhe and Chen, Zhuo and Chen, Ziyi
                   and Cong, Jian and Deng, Lelai and Ding, Chuang
                   and Gao, Lu and others},
  year          = {2024},
  eprint        = {2406.02430},
  archiveprefix = {arXiv},
  primaryclass  = {eess.AS}
}

@misc{meister2023librispeechpc,
  title         = {{LibriSpeech-PC}: Benchmark for Evaluation of Punctuation
                   and Capitalization Capabilities of End-to-End {ASR} Models},
  author        = {Meister, Aleksandr and Novikov, Matvei and Karpov, Nikolay
                   and Bakhturina, Evelina and Lavrukhin, Vitaly
                   and Ginsburg, Boris},
  year          = {2023},
  eprint        = {2310.02943},
  archiveprefix = {arXiv},
  primaryclass  = {eess.AS}
}

@misc{he2025openve,
  title         = {{OpenVE-3M}: A Large-Scale High-Quality Dataset for
                   Instruction-Guided Video Editing},
  author        = {He, Haoyang and Wang, Jie and Zhang, Jiangning
                   and Xue, Zhucun and Bu, Xingyuan and Yang, Qiangpeng
                   and Wen, Shilei and Xie, Lei},
  year          = {2025},
  eprint        = {2512.07826},
  archiveprefix = {arXiv},
  primaryclass  = {cs.CV}
}

@misc{minimax2026h3,
  title        = {{MiniMax H3}: An Open Model Breaking the Boundaries
                  Between Tasks and Modalities},
  author       = {{MiniMax}},
  year         = {2026},
  howpublished = {MiniMax Research Blog},
  note         = {Released July 31, 2026}
}

@misc{he2024emilia,
  title         = {Emilia: An Extensive, Multilingual, and Diverse Speech
                   Dataset for Large-Scale Speech Generation},
  author        = {He, Haorui and Shang, Zengqiang and Wang, Chaoren
                   and Li, Xuyuan and Gu, Yicheng and Hua, Hua
                   and Liu, Liwei and Yang, Chen and Li, Jiaqi
                   and Shi, Peiyang and Wang, Yuancheng and Chen, Kai
                   and Zhang, Pengyuan and Wu, Zhizheng},
  year          = {2024},
  eprint        = {2407.05361},
  archiveprefix = {arXiv},
  primaryclass  = {cs.SD}
}
\bibliographystyle{arxiv_preprint}

\clearpage
\appendix

\section{Reference Competition and Query-Side Discrimination}
\label{app:query-pressure}

This section gives the general form of the reference-competition argument
used in Section~\ref{sec:constraint}. Consider a target query $q$, one
instruction-relevant reference key $k_{+}$, and $N$ competing reference keys
$\{k_i\}_{i=1}^{N}$. Let their attention logits be

\begin{equation}
    a = \frac{q^\top k_{+}}{\sqrt d},
    \qquad
    b_i = \frac{q^\top k_i}{\sqrt d}.
\end{equation}

The attention probability assigned to the relevant token is

\begin{equation}
    p_{+}
    =
    \frac{e^a}
    {e^a+\sum_{i=1}^{N}e^{b_i}}.
\end{equation}

Requiring $p_{+}\geq\rho$ is equivalent to

\begin{equation}
    a-\log\sum_{i=1}^{N}e^{b_i}
    \geq
    \log\frac{\rho}{1-\rho}.
    \label{eq:general-query-pressure}
\end{equation}

Define the effective competing logit as

\begin{equation}
    b_{\mathrm{eff}}
    =
    \log
    \left(
        \frac{1}{N}
        \sum_{i=1}^{N}e^{b_i}
    \right).
\end{equation}

Since
$\log\sum_i e^{b_i}=\log N+b_{\mathrm{eff}}$,
Eq.~\eqref{eq:general-query-pressure} becomes

\begin{equation}
    a-b_{\mathrm{eff}}
    \geq
    \log N
    +
    \log\frac{\rho}{1-\rho}.
    \label{eq:effective-query-pressure}
\end{equation}

Thus, for a fixed desired attention mass $\rho$, the required separation
between the relevant token and the effective competing reference context grows
as $\log N$. When all competing logits are equal, $b_i=b$, the general
expression reduces to the equal-logit case.

The equal-logit case is therefore not required by the argument; it only gives
a convenient closed-form illustration. More generally, increasing the
reference context increases the amount of competing attention mass that must
be separated from instruction-relevant content.

Under an isolated cache, the reference keys $k_R(r)$ are independent of the
instruction, so instruction-specific discrimination must be expressed
primarily through the dynamic target query $q_X(X_t,c)$. AnchorCache instead
constructs instruction-conditioned reference keys $k_R(r,c)$, allowing the
conditioning signal to shape both sides of the retrieval interaction. This
does not remove reference competition, but provides additional representational
capacity for resolving it while preserving exact K/V reuse.

\section{Why the Recovery Is Staged}
\label{app:staged}

\paragraph{Stage 1: establishing a useful student trajectory.}
Converting a pretrained full-attention model to AnchorCache changes its
information flow and initially degrades generation quality. Stage~1 uses
data-derived interpolants to recover the student's prediction behavior
before relying on its own rollouts. Beyond improving generation quality,
this provides a useful initialization for Stage~2: the student can generate
reasonably coherent intermediate states at which the teacher is queried.

This initialization is important because on-policy distillation queries
the teacher at student-visited states rather than replacing the preceding
student trajectory with a teacher-generated one. If the student has already
accumulated large errors, its states may deviate substantially from those
encountered during successful generation. Although the teacher can still
predict a velocity at these states, matching that velocity locally does
not guarantee correction of errors accumulated earlier in the trajectory.
Stage~1 therefore helps establish a useful rollout distribution before
on-policy refinement.

\paragraph{Stage 2: refining student-visited states.}
Even after teacher-forced recovery, the student may encounter states that
differ from the data-derived interpolants used during Stage~1. Stage~2
addresses this train--inference mismatch by querying the full-attention
teacher at states visited along the student's own denoising trajectory.
The teacher and student are evaluated at the same detached states, so
the student receives supervision on its actual inference distribution
without backpropagating through the rollout. As the student is updated,
subsequent rollouts expose the teacher to the student's evolving states.

\paragraph{Motivation for the training order.}
The two stages thus serve complementary purposes: Stage~1 recovers a useful
initial generation trajectory, while Stage~2 targets residual errors along
student-visited states. Empirically, continued teacher-forced training
plateaus, whereas switching to on-policy refinement yields further
improvements (Figure~\ref{fig:recovery_curves}). This motivates our staged
schedule without implying that teacher-forced initialization is
mathematically necessary for on-policy distillation.

\section{Compatibility with Common Diffusion Acceleration Techniques}
\label{app:complementary-acceleration}

We evaluate whether AnchorCache can be combined with commonly used
diffusion acceleration techniques. Specifically, we apply
SageAttention2~\citep{zhang2025sageattention2} to AnchorCache and evaluate
the combined method on Qwen-Image-Edit (20B) using OmniContext.
AnchorCache removes repeated reference-side computation across denoising
steps, while remaining compatible with additional inference acceleration.
This experiment assesses their practical complementarity.

\paragraph{Evaluation.}
Table~\ref{tab:anchorcache-sageattention2} compares the full-attention
baseline, the baseline with SageAttention2, AnchorCache, and AnchorCache
with SageAttention2. We report prompt following (PF), subject consistency
(SC), and Overall scores, following the evaluation protocol described in
Appendix~\ref{app:details}. The first three rows reproduce the corresponding
results from Table~\ref{tab:main}. All speedups are relative to the same
full-attention baseline.

\begin{table}[!htbp]
  \centering
  \caption{
    \textbf{Combining AnchorCache with SageAttention2.}
    Quality is evaluated on OmniContext, with \emph{All 400} covering
    the complete benchmark. Speedup is measured over the complete
    40-step DiT denoising process with 8 reference images, relative to
    the full-attention baseline.
  }
  \label{tab:anchorcache-sageattention2}
  \small
  \setlength{\tabcolsep}{3.4pt}
  \renewcommand{\arraystretch}{1.20}

  \resizebox{\linewidth}{!}{%
  \begin{tabular}{@{}l cc cc ccc c@{}}
    \toprule
    & \multicolumn{2}{c}{\textsc{Multiple}}
    & \multicolumn{2}{c}{\textsc{Scene}}
    & \multicolumn{3}{c}{\emph{All 400}}
    & Efficiency \\
    \cmidrule(lr){2-3}
    \cmidrule(lr){4-5}
    \cmidrule(lr){6-8}
    \cmidrule(lr){9-9}

    Method
      & PF $\uparrow$ & SC $\uparrow$
      & PF $\uparrow$ & SC $\uparrow$
      & PF $\uparrow$ & SC $\uparrow$ & Overall $\uparrow$
      & Speedup $\uparrow$ \\
    \midrule

    Full-attention baseline
      & 8.660 & 8.020
      & 8.180 & 7.693
      & 8.550 & 7.867 & 8.119
      & 1.00$\times$ \\

    \quad + SageAttention2
      & 8.720 & 8.280
      & 7.993 & 7.580
      & 8.495 & 7.923 & 8.105
      & 1.47$\times$ \\

    \midrule

    AnchorCache
      & 8.573 & 8.453
      & 7.913 & 7.947
      & 8.418 & 8.167 & \textbf{8.185}
      & 6.40$\times$ \\

    \quad + SageAttention2
      & 8.41 & 8.32
      & 7.69 & 7.59
      & 8.31 & 8.01 & 8.062
      & \textbf{7.99$\times$} \\

    \bottomrule
  \end{tabular}%
  }
\end{table}

\paragraph{Additional acceleration.}
Combining AnchorCache with SageAttention2 increases the speedup from
$6.40\times$ to $7.99\times$ relative to the full-attention baseline.
This corresponds to an additional $1.25\times$ speedup over AnchorCache
alone, or approximately $19.9\%$ lower latency.
The combined method has a measured latency of $130.4$ seconds under
this setting.

\paragraph{Quality--efficiency trade-off.}
The additional acceleration is accompanied by a decrease in the Overall
score from $8.185$ to $8.062$.
Applying SageAttention2 to the full-attention baseline similarly improves
efficiency while reducing the Overall score, from $8.119$ to $8.105$.
The two settings therefore exhibit the same qualitative direction of
the trade-off, although the quality reduction is larger when
SageAttention2 is combined with AnchorCache
($0.123$ versus $0.014$ points).
These results demonstrate that AnchorCache and SageAttention2 can be
combined to obtain an additional quality--efficiency operating point,
supporting their practical complementarity.

\section{Compute Scaling with Reference Context}
\label{app:cost}

We analyze how the computation of full attention, isolated caching, and
AnchorCache scales with the number of reference tokens. Let
$N_{\mathrm{R}}$ denote the number of reference tokens,
$N_{\mathrm{ST}}$ the number of static-text tokens, and
\begin{equation}
    S = N_X + N_{\mathrm{LT}}
\end{equation}
the number of live target and instruction tokens. We omit constant factors and
focus on the dominant token-dependent terms in a transformer block.

\paragraph{Full attention.}
Under full attention, all $N_{\mathrm{R}}+S$ tokens are updated at every
denoising step. Over $T$ denoising steps, the computation therefore scales as
\begin{equation}
\mathcal{C}_{\mathrm{full}}
=
O\!\left(
T(N_{\mathrm{R}}+S)d^2
+
T(N_{\mathrm{R}}+S)^2d
\right),
\label{eq:cost-full}
\end{equation}
where the first term represents token-wise projections and feed-forward
computation, and the second represents attention.

\paragraph{Isolated cache.}
Once the reference branch is isolated from the live target and instruction
states, its representations can be computed once and reused across all
denoising steps. Cache construction over the reference tokens costs
\begin{equation}
\mathcal{C}_{\mathrm{iso,pre}}
=
O\!\left(
N_{\mathrm{R}}d^2
+
N_{\mathrm{R}}^2d
\right).
\end{equation}

After cache construction, only the $S$ live tokens are updated. Their queries
still attend to the cached reference K/V, giving a repeated cost
\begin{equation}
\mathcal{C}_{\mathrm{live}}
=
O\!\left(
T S d^2
+
T S(N_{\mathrm{R}}+S)d
\right).
\end{equation}

The total isolated-cache cost is therefore
\begin{equation}
\mathcal{C}_{\mathrm{iso}}
=
O\!\left(
N_{\mathrm{R}}d^2
+
N_{\mathrm{R}}^2d
+
T S d^2
+
T S(N_{\mathrm{R}}+S)d
\right).
\label{eq:cost-isolated}
\end{equation}

\paragraph{AnchorCache.}
AnchorCache has the same live computation as the isolated cache. The only
additional computation comes from including the static instruction tokens
$\mathrm{ST}$ in the one-time cache-construction subgraph
$[\mathrm{ST},\mathrm{R}]$. Its total cost is
\begin{equation}
\mathcal{C}_{\mathrm{anchor}}
=
O\!\left(
(N_{\mathrm{R}}+N_{\mathrm{ST}})d^2
+
(N_{\mathrm{R}}+N_{\mathrm{ST}})^2d
+
T S d^2
+
T S(N_{\mathrm{R}}+S)d
\right).
\label{eq:cost-anchor}
\end{equation}

Importantly, Eqs.~\eqref{eq:cost-isolated} and
\eqref{eq:cost-anchor} have identical repeated per-step computation.
AnchorCache therefore introduces no additional reference-dependent cost during
the denoising trajectory.

The extra computation relative to the isolated cache is confined to cache
construction:
\begin{align}
\Delta\mathcal{C}_{\mathrm{anchor}}
&=
\mathcal{C}_{\mathrm{anchor}}
-
\mathcal{C}_{\mathrm{iso}}
\\
&=
O\!\left(
N_{\mathrm{ST}}d^2
+
\left(
2N_{\mathrm{R}}N_{\mathrm{ST}}
+
N_{\mathrm{ST}}^2
\right)d
\right).
\label{eq:anchor-overhead}
\end{align}

Since the static-text block is computed only once, this overhead is amortized
over the full denoising trajectory. Moreover,
$N_{\mathrm{ST}}=N_{\mathrm{LT}}$ in AnchorCache and is independent of the
number of references. Consequently, as the reference context grows, the
dominant saving remains the removal of repeated reference-side computation,
while the additional cost of instruction-aware cache construction remains a
one-time overhead.

\paragraph{Scaling with reference context.}
The key difference is the reference-dependent repeated computation. Full
attention repeatedly incurs quadratic reference-side attention,
\begin{equation}
O(TN_{\mathrm{R}}^2),
\end{equation}
whereas both cacheable variants pay the quadratic reference cost only once and
reduce the repeated reference-dependent term to
\begin{equation}
O(TSN_{\mathrm{R}}).
\end{equation}

AnchorCache therefore preserves the same asymptotic denoising-time complexity
as the isolated cache while adding only the one-time overhead in
Eq.~\eqref{eq:anchor-overhead}. This is consistent with the measured results:
AnchorCache achieves nearly the same inference speedup as the isolated cache
while substantially recovering generation quality.

This analysis describes structural compute scaling rather than exact
wall-clock latency. Actual latency additionally depends on memory traffic,
kernel-launch overhead, hardware utilization, and other model components.

\section{Experimental Details}
\label{app:details}

\subsection{OPD Query Sampling}
\label{app:opd-query}

We vary both the region of the denoising trajectory from which supervision
states are drawn and the number of query states $N_q$ collected from each
student rollout. All configurations are initialized from the same 30k Stage~1
checkpoint, whose OmniContext Overall score is $8.019$.

For the full-trajectory setting, query states are sampled across the complete
student denoising trajectory. The restricted settings use only the
corresponding portions of the 40-step rollout:

\begin{align}
    \text{High-noise:} \quad &k \in [1,14], \\
    \text{Mid-noise:} \quad  &k \in [15,24], \\
    \text{Low-noise:} \quad  &k \in [25,39], \\
    \text{Full trajectory:} \quad &k \in [1,39].
\end{align}

Here $k$ denotes the rollout step index. For each configuration, the
$N_q$ query states are sampled only from the specified interval.

The results in Table~\ref{tab:opd-query} show that increasing $N_q$ alone does
not guarantee better recovery. In particular, additional queries concentrated
in the low-noise region degrade performance, whereas sampling across the full
trajectory performs best for both evaluated values of $N_q$. We therefore use
full-trajectory sampling with $N_q=4$ as the default Stage~2 configuration.

\paragraph{Recovery hyperparameters.}
Table~\ref{tab:hparams} summarizes the settings used for the two recovery
stages.  We use a relatively tight gradient-norm
clip of $0.05$, which we found empirically stable for the small recovery
updates in this setting. Stage~2 is initialized from the Stage~1 checkpoint
rather than from the original pretrained model.

Both recovery stages train all parameters in bf16 with DeepSpeed ZeRO-2 and
gradient checkpointing, using a global batch size of 8, AdamW
($\beta=(0.9,0.999)$, $\varepsilon=10^{-10}$), cosine scheduling with 200
warmup steps, gradient clipping at $0.05$, and learning rate $5\times10^{-6}$.

\begin{table}[htbp]
  \centering
  \caption{\textbf{Recovery-stage configurations.}}
  \label{tab:hparams}
  \small
  \setlength{\tabcolsep}{7pt}
  \renewcommand{\arraystretch}{1.15}

  \begin{tabular}{@{}lcc@{}}
    \toprule
    & \textbf{Stage 1} & \textbf{Stage 2} \\
    & Teacher-forced & On-policy \\
    \midrule
    Weight decay
        & $3\times10^{-2}$
        & $1\times10^{-2}$ \\

    Training steps
        & $30$k
        & $500$ \\

    Initialization
        & Pretrained model
        & Stage~1 checkpoint \\

    Student rollout
        & ---
        & $40$ steps, no CFG \\

    Query states
        & Data-derived interpolants
        & $N_q=4$, full trajectory \\

    \bottomrule
  \end{tabular}
\end{table}

\paragraph{Recovery Training Data.}
For OmniImage Generation, recovery training uses a mixture of
OmniEdit~\citep{wei2025omniedit} and
OrionEdit~\citep{jiang2026orionedit}.
Both recovery stages use the same data mixture.
For zero-shot TTS, recovery training uses the
Emilia-Dataset~\citep{he2024emilia}.
For OmniVideo Generation, recovery training uses
OpenVE-3M~\citep{he2025openve}.

\paragraph{Sampling.}
All OmniContext images are generated using the same settings across methods:
$40$ denoising steps, $1024\times1024$ reference and output resolution,
classifier-free guidance with $\gamma=4$ and a blank negative prompt, and the
same fixed seed ($42$) for all $400$ examples. GEdit-Bench follows its released
evaluation protocol ($50$ steps, $\gamma=4$, and per-example seeds), so its
absolute values should not be compared directly with OmniContext scores.

\paragraph{Scoring.}
We use the released evaluation code and prompt format of each benchmark without
modification. For OmniContext, we follow its VIEScore-style evaluation with
GPT-4.1 as the judge.

\section{What Does Not Work}
\label{app:negative}

Before settling on pointwise on-policy regression, we also evaluated
guidance-aware Stage~2 variants. We report these experiments because the
deployed sampler uses classifier-free guidance, making it natural to ask
whether Stage~2 should use the same guided trajectory and branch structure.

\paragraph{Guided rollouts and branch-aware losses.}
\label{app:opd-cfg}

At inference, prediction uses classifier-free guidance

\begin{equation}
    v
    =
    v^{-}
    +
    \gamma(v^{+}-v^{-}),
    \qquad \gamma=4.
\end{equation}

Our default recovery procedure differs: student rollout, teacher evaluation,
and regression all use the conditional branch without CFG. CFG is introduced
only during final inference.

We additionally evaluate two branch-aware objectives under a guided student
rollout. Following \citet{li2026rethinkingcfgopd}, independent branch matching
(IBM) minimizes

\begin{equation}
    \|v^{+}_{\theta}-v^{+}_{\mathrm{full}}\|^2
    +
    \|v^{-}_{\theta}-v^{-}_{\mathrm{full}}\|^2,
\end{equation}

while positive--direction matching (PDM) matches the positive branch together
with the guidance direction
$d=v^{+}-v^{-}$, using $\lambda=1$.
\begin{table}[htbp]
  \centering
  \caption{
  \textbf{Comparison of OPD supervision variants.}
  We compare unguided OPD with two guidance-based variants on OmniContext.
  $\Delta$ denotes the change in Overall score relative to default OPD.
  }
  \label{tab:opd-guidance}

  \small
  \setlength{\tabcolsep}{5pt}
  \renewcommand{\arraystretch}{1.15}

  \begin{tabular}{@{}lllcccc@{}}
    \toprule
    Variant & Rollout & Supervision
    & PF $\uparrow$ & SC $\uparrow$
    & Overall $\uparrow$ & $\Delta$ \\
    \midrule

    Default OPD
      & unguided
      & $v^{+}$ only
      & $8.418$ & $8.167$ & $\mathbf{8.185}$ & --- \\

    IBM
      & guided, $\gamma{=}4$
      & $v^{+}$ and $v^{-}$
      & $8.283$ & $7.987$ & $8.025$ & $-0.160$ \\

    PDM
      & guided, $\gamma{=}4$
      & $v^{+}$ and $d$
      & $8.223$ & $8.040$ & $8.021$ & $-0.164$ \\

    \bottomrule
  \end{tabular}
\end{table}

Here $\Delta$ is measured relative to the default OPD configuration.
Neither tested guidance-aware configuration improves over the default
unguided Stage~2 procedure.

These experiments do not isolate the effect of guided rollout from that of
branch-aware supervision, because a guided-rollout, positive-only control was
not evaluated. The result should therefore be interpreted as a comparison of
the complete configurations above rather than as evidence against either
component individually.

One possible explanation, consistent with the analysis of
\citet{li2026rethinkingcfgopd}, is that our setting does not exhibit the same
privileged-negative asymmetry. Teacher and student receive the same reference
inputs in the negative branch and differ primarily in attention structure.
Under such a setting, the cross-branch compensation targeted by branch-aware
objectives may be less pronounced. We leave a controlled decomposition of
trajectory guidance and branch supervision to future investigation.


\section{Additional Results}
\label{app:additional-results}


\subsection{Contextualizing the Full-Attention Base Model}
\label{app:external}

Table~\ref{tab:external} provides additional context for the full-attention
base model by evaluating several recent instruction-based image-editing
systems on GEdit-Bench under a common inference and scoring protocol.
These systems differ in model scale, training data, architecture, and
preferred inference configuration; the table is therefore not intended
as a controlled comparison or leaderboard.

All rows are generated and scored by us using the same evaluation pipeline,
including the same judge version, prompt format, sampler configuration,
resolution, and absence of per-model tuning. The resulting absolute scores
may therefore differ from values reported in the corresponding papers,
including for our own base model.

\begin{table}[!htb]
  \caption{
  Recent instruction-based image-editing systems evaluated by us on
  GEdit-Bench under a common inference and scoring protocol.
  The table provides additional context for our full-attention base model
  rather than a controlled leaderboard. 
  $^{\ast}$FireRed-Image-Edit is evaluated at version 1.1, while the cited
  report describes version 1.0.
  }
  \label{tab:external}

  \centering
  \small
  \setlength{\tabcolsep}{6pt}
  \renewcommand{\arraystretch}{1.12}

  \begin{tabular}{@{}lccc@{}}
    \toprule
    Model
      & SC $\uparrow$
      & PQ $\uparrow$
      & Overall $\uparrow$ \\
    \midrule

    FireRed-Image-Edit-1.1$^{\ast}$ \citep{firered2026imageedit}
      & 8.252
      & 7.503
      & 7.762 \\

    \textbf{Qwen-Image-Edit-2511 (our base)}
      \citep{wu2025qwenimage}
      & 8.068
      & 7.345
      & 7.580 \\

    LongCat-Image-Edit \citep{longcat2025image}
      & 8.041
      & 7.358
      & 7.554 \\

    Boogu-Image-0.1-Edit \citep{boogu2026image}
      & 7.784
      & 7.021
      & 7.215 \\

    JoyAI-Image-Edit \citep{joyai2026image}
      & 7.569
      & 7.271
      & 7.134 \\

    FLUX.2 [dev] \citep{flux2dev2025}
      & 7.384
      & 7.297
      & 7.073 \\

    \bottomrule
  \end{tabular}
\end{table}


\subsection{Qualitative Examples on OpenVE-Bench}
\label{app:openve-qualitative}

Figure~\ref{fig:openve-qual} shows representative video-editing results
generated by AnchorCache on OpenVE-Bench across different editing scenarios.

\par\medskip
\noindent
\begin{minipage}{\linewidth}
  \centering

  \includegraphics[
      width=\linewidth
  ]{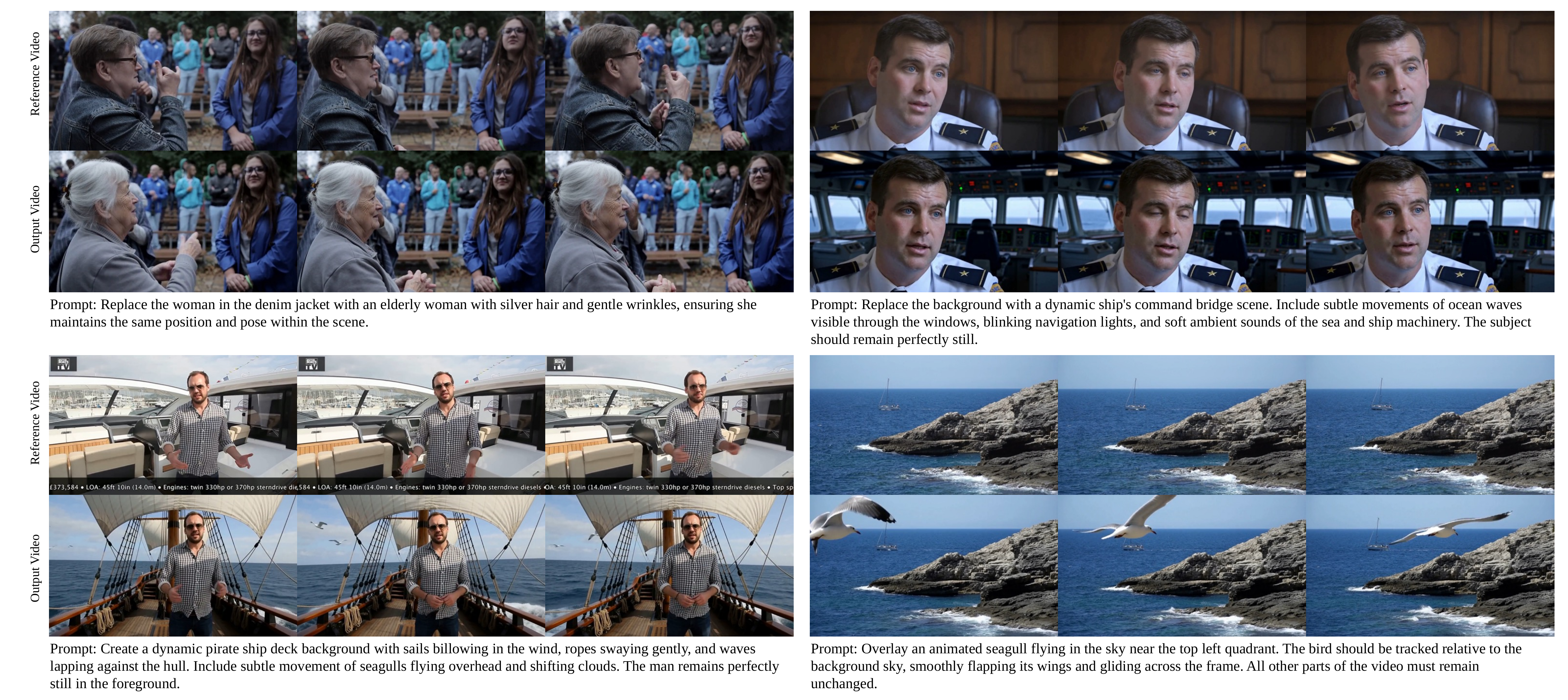}

  \vspace{-6pt}
  \captionsetup{hypcap=false}
  \captionof{figure}{
    \textbf{Qualitative examples on OpenVE-Bench.}
    Representative video-editing results generated by AnchorCache across
    different editing scenarios.
  }
  \label{fig:openve-qual}

\end{minipage}


\subsection{A Batch of GEdit-Bench Examples}
\label{app:gedit-qualitative}

Figures~\ref{fig:gedit-qual-a} and~\ref{fig:gedit-qual-b} show ten contiguous
GEdit-Bench examples using the same four-column layout: input, full attention,
isolated cache, and AnchorCache. Unlike the selected qualitative examples in
the main paper, these figures show a contiguous batch and therefore also
include cases for which the three model outputs are visually similar.


\par\medskip
\noindent
\begin{minipage}{\linewidth}
  \centering

  \includegraphics[
      width=\linewidth
  ]{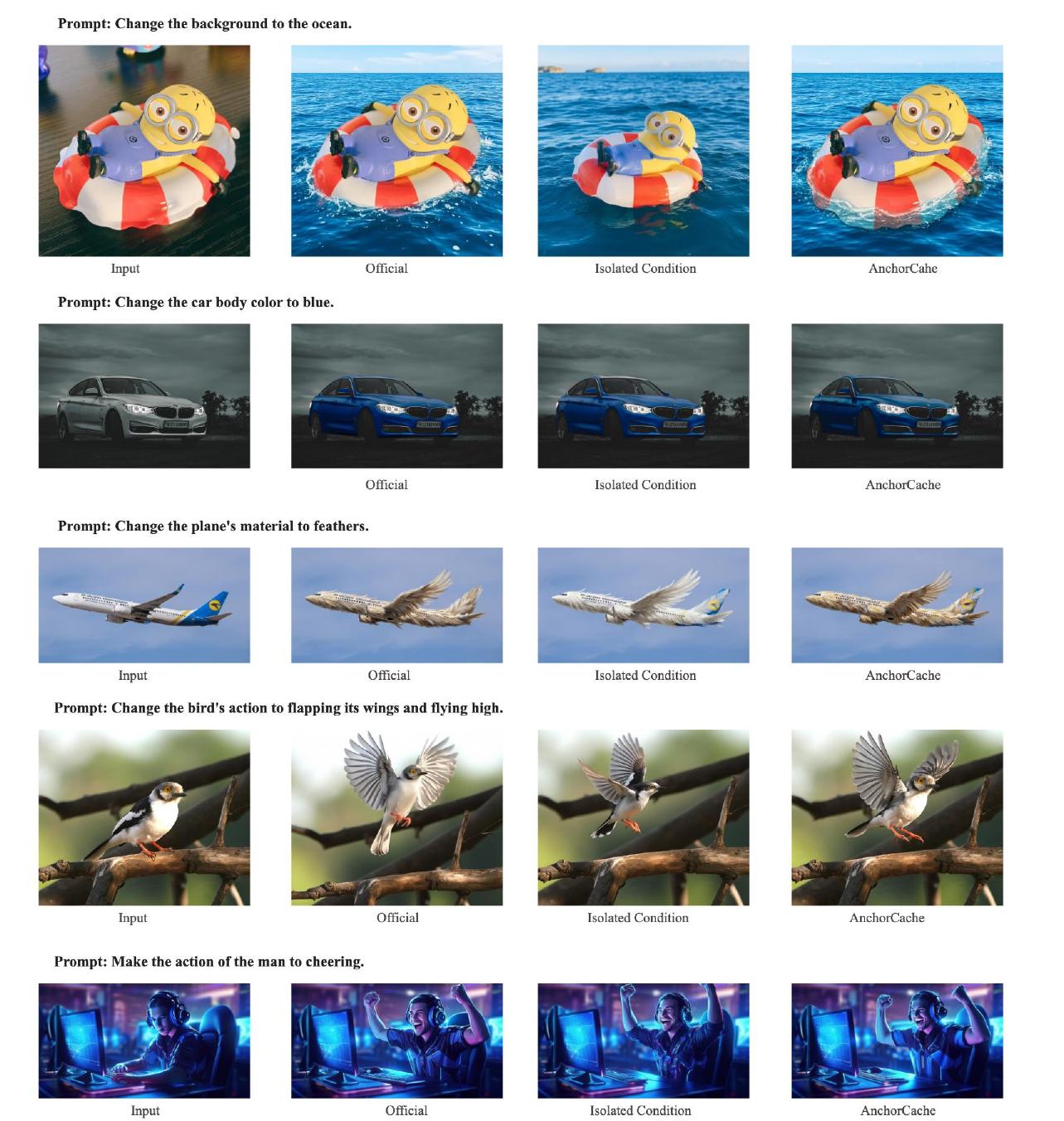}

  \vspace{-6pt}
  \captionsetup{hypcap=false}
  \captionof{figure}{
    \textbf{Qualitative examples on GEdit-Bench, part 1 of 2.}
    Five examples from the same contiguous batch.
  }
  \label{fig:gedit-qual-a}

\end{minipage}


\newpage

\noindent
\begin{minipage}{\linewidth}
  \centering

  \includegraphics[
      width=\linewidth
  ]{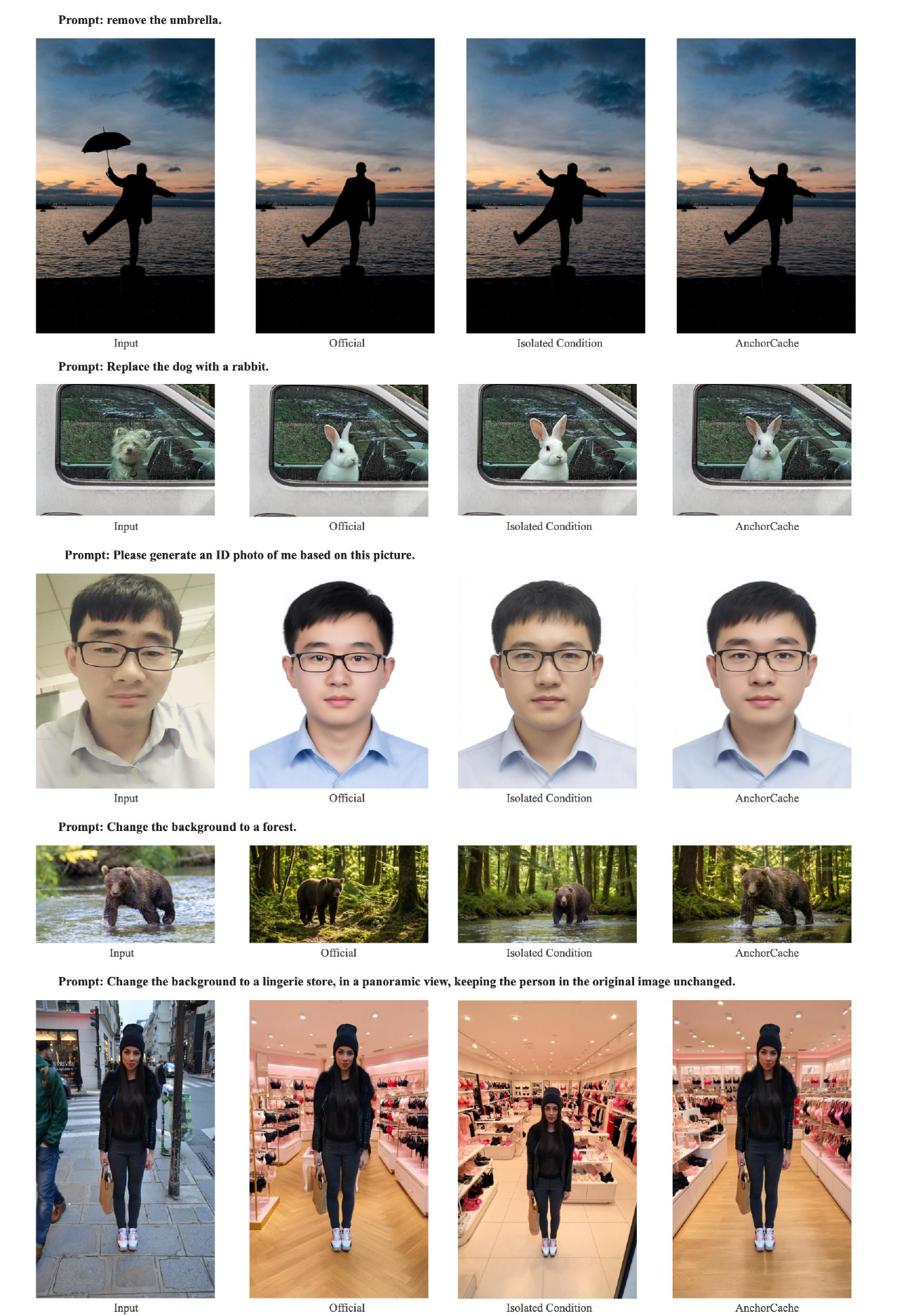}

  \vspace{-6pt}
  \captionsetup{hypcap=false}
  \captionof{figure}{
    \textbf{Qualitative examples on GEdit-Bench, part 2 of 2.}
    Five additional examples from the contiguous batch.
    Columns show the input, full-attention output, isolated-cache output,
    and AnchorCache output. The editing instruction is shown above each row.
  }
  \label{fig:gedit-qual-b}

\end{minipage}

\end{document}